\documentclass[sigconf, nonacm]{aamas}

\usepackage{balance} % for balancing columns on the final page
\usepackage{enumitem}% for item
\usepackage{utfsym}% for usym symbol
\usepackage[linesnumbered,ruled,vlined]{algorithm2e}% for algorithm
\usepackage[section]{placeins}
\usepackage{makecell}
\usepackage{amsmath}
\usepackage{subfigure}
\usepackage{graphicx}
\usepackage{verbatim}
\usepackage{appendix}
\usepackage{epstopdf}
\usepackage{caption}
\usepackage{float} 
\usepackage{subcaption}
\usepackage{stfloats}
\usepackage{hyperref}

\setcopyright{ifaamas}
\acmConference[AAMAS '25]{Proc.\@ of the 24th International Conference
on Autonomous Agents and Multiagent Systems (AAMAS 2025)}{May 19 -- 23, 2025}
{Detroit, Michigan, USA}{A.~El~Fallah~Seghrouchni, Y.~Vorobeychik, S.~Das, A.~Nowe (eds.)}
\copyrightyear{2025}
\acmYear{2025}
\acmDOI{}
\acmPrice{}
\acmISBN{}

\acmSubmissionID{<<644>>}

\title[AAMAS-2025 Formatting Instructions]{FedHPD: Heterogeneous Federated Reinforcement Learning via Policy Distillation}

\author{Wenzheng Jiang}
\affiliation{
  \institution{National University of Defense Technology}
  \country{China}}
\email{jiangwenzheng@nudt.edu.cn}

\author{Ji Wang}
\affiliation{
  \institution{National University of Defense Technology}
  \country{China}}
\email{wangji@nudt.edu.cn}

\author{Xiongtao Zhang}
\affiliation{
  \institution{National University of Defense Technology}
  \country{China}}
\email{zhangxiongtao14@nudt.edu.cn}

\author{Weidong Bao}
\affiliation{
  \institution{National University of Defense Technology}
  \country{China}}
\email{wdbao@nudt.edu.cn}

\author{Cheston Tan}
\affiliation{
  \institution{CFAR, A*STAR}
  \country{Singapore}}
\email{cheston-tan@i2r.a-star.edu.sg}

\author{Flint Xiaofeng Fan}
\affiliation{
  \institution{National University of Singapore}
  \country{Singapore}}
\email{fxf@u.nus.edu}

\begin{abstract}
\textit{Federated Reinforcement Learning} (FedRL) improves sample efficiency while preserving privacy; however, most existing studies assume homogeneous agents, limiting its applicability in real-world scenarios. This paper investigates FedRL in black-box settings with heterogeneous agents, where each agent employs distinct policy networks and training configurations without disclosing their internal details. \textit{Knowledge Distillation} (KD) is a promising method for facilitating knowledge sharing among heterogeneous models, but it faces challenges related to the scarcity of public datasets and limitations in knowledge representation when applied to FedRL. To address these challenges, we propose \textit{Federated Heterogeneous Policy Distillation} (FedHPD), which solves the problem of heterogeneous FedRL by utilizing action probability distributions as a medium for knowledge sharing. We provide a theoretical analysis of FedHPD's convergence under standard assumptions. Extensive experiments corroborate that FedHPD shows significant improvements across various reinforcement learning benchmark tasks, further validating our theoretical findings. Moreover, additional experiments demonstrate that FedHPD operates effectively without the need for an elaborate selection of public datasets.
\end{abstract}


\newcommand{\BibTeX}{\rm B\kern-.05em{\sc i\kern-.025em b}\kern-.08em\TeX}

\newtheorem{assumption}{Assumption}
\newtheorem{proposition}{Proposition}
\newtheorem{theorem}{Theorem}
\newtheorem{lemma}{Lemma}
\newtheorem{corollary}{Corollary}
\newtheorem*{remark}{Remark}

\begin{document}

%%% The following commands remove the headers in your paper. For final 
%%% papers, these will be inserted during the pagination process.

\pagestyle{fancy}
\fancyhead{}

%%% The next command prints the information defined in the preamble.

\maketitle 

%%%%%%%%%%%%%%%%%%%%%%%%%%%%%%%%%%%%%%%%%%%%%%%%%%%%%%%%%%%%%%%%%%%%%%%%

\section{Introduction}
In recent years, \textit{Reinforcement Learning} (RL) has achieved outstanding performance in various tasks \cite{ju2022transferring,xin2022supervised,chai2022design,feng2023dense}. However, it still faces low sample efficiency \cite{yu2018towards} and privacy leakage \cite{pan2019you}. 
\textit{Federated Learning} (FL) \cite{mcmahan2017communication} enables clients to collaboratively improve training efficiency while preserving data privacy, and mitigates instabilities associated with centralized model updates.
The collaborative privacy-preserving characteristics of FL, combined with the need for sample-efficient and private training in RL, have led to the emergence of \textit{Federated Reinforcement Learning} (FedRL).
This fusion offers a promising approach for intelligent decision-making in distributed privacy-sensitive environments~\cite{DAI2024257}.
Recognizing its potential, the research community has extensively explored FedRL under various settings, such as Byzantine fault-tolerance~\cite{fan2021fault,feng2023robust}, environment heterogeneity~\cite{jin2022federated,woo2023blessing,zhang2024fedsarsa,mak2024caesar}, and decentralization~\cite{ETH.DECEN.BYZPG,qiao2024br}.

Despite its promise, most FedRL frameworks \cite{fan2021fault,jin2022federated,zhang2024fedsarsa,fan2024fedrlhf} operate under the assumption of agent homogeneity (i.e., identical policy networks and training configurations), which significantly limits FedRL's applicability in real-world scenarios. 
This limitation is particularly acute in resource-constrained environments, such as in edge environments, where agents have limited power and need to adapt network structures and training strategies based on their operational conditions to achieve effective training \cite{ye2023heterogeneous}. 
In addition, existing FedRL frameworks typically operate under a white-box paradigm, where models are openly shared among participants.
However, in many business-oriented domains, such as healthcare and finance, the disclosure of internal details to server is often prohibited due to commercial sensitivities, intellectual property concerns, and regulatory compliance.
% For example, in fields like healthcare, finance, and law, institutions may be unwilling to disclose their details during collaboration due to certain commercial purposes or intellectual property concerns.
%Industries such as healthcare, finance, and law frequently restrict such information sharing due to commercial sensitivities, intellectual property concerns, or regulatory compliance.
These practical constraints lead us to formulate a more challenging question:
% \textbf{\textit{When each agent possesses a different model that is a black-box to others, how to perform FedRL?}}

\textbf{\textit{How can we effectively perform FedRL when each agent employs a unique model that remains a black box to all other participants?}}

Compared with previous FedRL frameworks, the challenges faced by agent heterogeneity\footnote{In this paper, ``agent heterogeneity'' indicates that policy networks and training configurations of agents in FedRL are heterogeneous \cite{FHQL2023}.} in black-box settings discussed in this research can be revealed as follows:
\begin{itemize}[leftmargin=10pt]
    \item \textbf{Knowledge Representation Difference.} 
    The lack of disclosure of agents' internal details results in a distributed black-box optimization problem \cite{cuccu2022dibb}.
    Traditional federated aggregation methods (such as FedAvg \cite{mcmahan2017communication}) are not applicable due to the agent heterogeneity. 
    Therefore, novel approaches are required to convert and align different models to facilitate effective knowledge sharing.
    \textit{Knowledge Distillation} (KD) \cite{hinton2015distilling} is regarded as an antidote to model heterogeneity \cite{li2019fedmd,lin2020ensemble,zhu2021data}.
    However, existing FedRL methods related to KD are often accompanied by several limitations, such as model-based \cite{ryu2022model} and partial network sharing \cite{mai2023server}, which hinders their applications in our settings.
    \item \textbf{Inconsistent Learning Progress.} 
    The different configurations of agents lead to increased complexity in federated optimization. 
    The inconsistent nature of learning requires effective balancing of knowledge from different agents, otherwise may impact the original convergence process. 
    While existing studies have provided proofs for convergence acceleration in homogeneous settings within FedRL \cite{khodadadian22a,woo2023blessing,zheng2024federated}, such proofs remain a topic of debate in heterogeneous contexts.
\end{itemize}

To the best of our knowledge, two existing studies closely resemble our work in addressing the challenges posed by agent heterogeneity. 
Fan et al. \cite{FHQL2023} proposed FedHQL to address the information inconsistency through inter-agent exploration.
However, FedHQL's reliance on server-side Markov Decision Process (MDP) and frequent state queries to local agents results in high communication overhead and latency, thereby limiting its practicality in bandwidth-constrained scenarios or systems with a large number of participants.
Jin et al. \cite{10609408} proposed DPA-FedRL to investigate privacy issues with multi-heterogeneity.
However, it still relies on the assumption that the server can query black-box agents at any time, and evaluating the global policy does not effectively reflect the sample efficiency improvement for heterogeneous individuals.

In light of the above challenges, we propose \textit{Federated Heterogeneous Policy Distillation} (FedHPD).
A pivotal feature of FedHPD is its elimination of reliance on the server-side MDP during training.
Unlike traditional policy distillation \cite{rusu2015policy}, FedHPD periodically extracts knowledge from local policies to form a global consensus at the server, and then distributes it to local agents for policy updates.
Through distillation, we achieve knowledge alignment among heterogeneous agents.
Compared to querying given states, periodic distillation ensures the continuity and stability of local training and helps to improve policy generalization.
Furthermore, instead of Q-learning methods used in FedHQL, our approach leverages policy gradient techniques, which mitigates value estimation bias and enhances performance in complex environments. 
Through the alternating process of multiple rounds of local training and periodic collaborative training, FedHPD is well-suited to tackle the inconsistent learning process of heterogeneous FedRL across various real-world applications.

\begin{table}[t]
    \caption{Comparison of FedRL methods. AH represents agent heterogeneity; MLU represents multiple local updates; ISM represents independence from server MDP; SEI represents sample efficiency improvement and MF represents model-free.}
	\label{tab:comparison-frl-methods}
    \vspace{-0.2cm}
	\begin{tabular}{cccccc}\toprule
		\textit{Method} & \textit{AH} & \textit{MLU} & \textit{ISM} & \textit{SEI} & \textit{MF} \\ \midrule
		%FedPG-BR\citep{fan2021fault} & \usym{2717} & \usym{2717} & \usym{2713} & \usym{2713} & \usym{2713}\\
		PAVG\citep{jin2022federated} & \usym{2613} & \usym{2714} & \usym{2714} & \usym{2714} & \usym{2714} \\
        FedSARSA\citep{zhang2024fedsarsa} & \usym{2613} & \usym{2714} & \usym{2714}& \usym{2714} & \usym{2714} \\
        Model-based FRD\citep{ryu2022model} & \usym{2613} & \usym{2613} & \usym{2714}& \usym{2613} & \usym{2613} \\
        SCCD\citep{mai2023server} & \usym{2613} & \usym{2714} & \usym{2714}& \usym{2714} & \usym{2714} \\
    	FedHQL\citep{FHQL2023} & \usym{2714} & \usym{2613} & \usym{2613} & \usym{2714} & \usym{2714}\\
        DPA-FedRL\citep{10609408} & \usym{2714} & \usym{2613} & \usym{2613} & \usym{2613} & \usym{2714}\\
    	\textbf{FedHPD (ours)} & \usym{2714} & \usym{2714} & \usym{2714} & \usym{2714} & \usym{2714} \\ \bottomrule
	\end{tabular}
\end{table}

%More concretely, we propose using \textit{knowledge distillation} (KD) \cite{hinton2015distilling} to achieve knowledge sharing between heterogeneous agents and the action probability distribution output by the policy network as the form of knowledge representation.
%In FL, KD is widely used to address the model heterogeneity \cite{li2019fedmd,lin2020ensemble,zhu2021data}, enabling knowledge transfer through the models' outputs on the public dataset.
%However, in RL, there is no public dataset available for distillation.
%To address this, we generate an offline public state set as shared input data using a virtual agent.
%Notably, we separate the generation of state set from the training of agents, with local agents all trained using the easily deployable on-policy REINFORCE algorithm.
%We set up multiple rounds of local updates and periodically implement collaborative training, completely getting rid of the MDP process and on-demand querying mechanism on the server side in FedHQL.

We situate FedHPD with respect to recent advancements of FedRL in Table \ref{tab:comparison-frl-methods}.
The main contributions can be summarized as follows:
\begin{itemize}[leftmargin=10pt]
\item We propose a novel FedRL framework named FedHPD, which firstly introduces KD to address the federated heterogeneous policy gradient problem. By periodically distilling the local policies, it effectively enhance system performance and individual sample efficiency (Section \ref{sec 4 fedhpg}).
\item As a starting point, we prove the convergence of the policy gradient with KD under standard assumptions. 
Building on this, we prove the fast convergence of FedHPD. 
To the best of our knowledge, this is the first theoretically principled FedRL framework with heterogeneous agents (Section \ref{sec.5 Theoretical Analysis}).
\item We extensively evaluate the performance of FedHPD on both discrete and continuous action space tasks, and investigate the impact of the distillation interval to validate our theory (Section \ref{sec.6 Experiments}). 
\end{itemize}

\section{Preliminaries}\label{sec 3 Preliminaries}
\subsection{Markov Decision Process}
Almost all RL problems can be modeled as Markov Decision Processes (MDPs) \cite{sutton2018reinforcement}: $\mathcal{M}$ = $\{\mathcal{S},\mathcal{A},\mathcal{P},\mathcal{R},\gamma,\rho\}$, where $\mathcal{S}$ and $\mathcal{A}$ represent the state and action spaces, ${\mathcal{P}}\left(s^{\prime}\mid s,a\right)$ denotes the transition probability from state $s$ to state $s'$ when taking action $a\in\mathcal{A}$. 
$\mathcal{R}\left(s,a\right):\mathcal{S}\times\mathcal{A}\to\left[0,\mathcal R_{\max}\right]$ is the reward function for taking action $a$ in state $s$, where $\mathcal R_{\max}>0$.
$\gamma\in\left(0,1\right)$ and $\rho$ respectively denote the discount factor and the initial state distribution.

An agent's actions are determined by a policy $\pi$, where $\pi\left(a|s\right)$ denotes the probability for state-action pair $\left(s; a\right)$. A trajectory $\tau\triangleq\{s_{0},a_{0},s_{1},a_{1},\ldots,s_{H-1},a_{H-1},s_{H}\}$ is collected through the interaction with the environment, where $H$ is the task horizon and $s_{0}\sim\rho$. We define $\mathcal{R}\left(\tau\right)\triangleq \sum_{t=0}^{H-1}\gamma^{t}\mathcal{R}\left(s_{t},a_{t}\right)$ as the cumulative reward of this trajectory.

\subsection{Policy Gradient}
Policy Gradient methods \cite{pmlr-v37-schulman15,schulman2017proximal,pmlr-v124-touati20a,pmlr-v151-yuan22a} have demonstrated remarkable success in the model-free RL, particularly showcasing effectiveness in high-dimensional problems and offering the flexibility of stochasticity. In PG, we denote the policy parameterized by $\theta$ as $\pi_{\theta}$, where $\theta$ represents the policy parameters, such as the weights of a neural network. The performance of a policy $\pi_{\theta}$ can be quantified by the expected return $J({\theta})$, defined as: $J(\theta)=\mathbb{E}_{\tau\sim\pi_\theta}\left[\mathcal{R}(\tau)\right]$.

To optimize the policy, we compute the gradient of $J(\boldsymbol{\theta})$ with respect to $\theta$:
\begin{align}\label{pg reinforce}
\nabla_{\theta}J(\theta)=\mathbb{E}_{\tau\sim\pi_\theta}\left[\sum_t\nabla_\theta\log\pi_\theta(a_t|s_t)\mathcal{R}(\tau)\right].
\end{align}

Then the policy parameters $\theta$ can be updated via gradient ascent. 
REINFORCE \cite{sutton2018reinforcement} is a fundamental on-policy RL algorithm that does not require trajectory sampling.
The update of the REINFORCE is: $\theta \leftarrow \theta + \alpha \nabla_\theta J(\theta)$, where $\alpha$ is the learning rate.

\subsection{Knowledge Distillation}
Knowledge Distillation (KD) \cite{hinton2015distilling} is a model compression technique that achieves model learning by minimizing the difference between the outputs of the student model and the teacher model.
The Kullback-Leibler divergence (KL divergence), a common measure of the difference between two probability distributions, is used in KD to quantify the difference between the teacher and student model outputs' probability distributions:
\begin{equation}
\min_{\theta_s}\mathbb{E}_{z\sim D_p}[D_{KL}(p(z;\theta_s)||p(z;\theta_t))],
\end{equation}
where $D_p$ is the proxy dataset, $p(z;\theta_t)$ and $p(z;\theta_s)$ denote the output of the teacher model and the student model, $\theta_t$ and $\theta_s$ are the parameters of the student and teacher models, respectively. 

KD has been extended to the domain of FL to address the model heterogeneity, by treating each client model as the teacher, whose information is aggregated into the student (global) model to improve its generalization performance:
\begin{equation}
\min_{\theta_s}\mathbb{E}_{z\sim D_p}[D_{KL}(p(z;\theta_s)||\frac1K{\sum_{k=1}^K}p(z;\theta_k))],
\end{equation}
where $\theta_k$ represents the parameters of the $k$-th client model.

\section{Federated Heterogeneous Policy Gradient}\label{sec 4 fedhpg}

\subsection{Problem Formulation}\label{sec.4.1}
Assuming the system contains a set of $K$ heterogeneous and mutually black-box agents, where agent $k$ interacts in a separate copy of the MDP $\mathcal{M}$ according to policy $\pi_k$, generating their own local private data ${D_{k}}\triangleq\{(s,a,s^{\prime},r)_{i}\}_{i=1}^{|{D_{k}}|}$.
The policy $\pi_k(a|\phi_k(s;\theta_k))$ is composed of a nonlinear function $\phi_k(s;\theta_k)$ that predicts the probability of taking action $a$ given the state $s$.
The nonlinear function $\phi_k(s;\theta_k)$ is parameterized by a set of parameters $\theta_k$ and is learned using private experience data $D_{k}$.

The nonlinear function is generally approximated using neural networks, and the agent heterogeneity is often manifested in the differences in neural network architectures and training configurations \citep{FHQL2023}.
The key to addressing agent heterogeneity is to find effective methods to enable information sharing among agents. 
We aim to develop a new FedRL framework that allows heterogeneous agents to share their knowledge from local policies in a black-box manner, without being constrained by the server-side MDP process. 
The objective functions of the federated heterogeneous policy gradient can be formulated as follows:

\textit{Assuming that all agents are trained synchronously, given the number of training rounds $T$, let $\psi_k^{\prime}$ and $\psi_k$ represent the training performance of agent $k$ under the FedRL framework and its performance when trained independently, respectively. We have:}

\noindent\textit{Under the same training rounds,}

\textit{for system:}
\begin{equation}
{\sum_{k=1}^K}\psi_k^{\prime} \geq {\sum_{k=1}^K}\psi_k;\label{System object}
\end{equation}

\textit{for agent $k, k\in\{1, 2, \ldots, K\}$:}
\begin{equation}
\psi_k^{\prime} \geq \psi_k. \label{Individual object}
\end{equation}

\subsection{Technical Challenges}
KD can effectively address the model heterogeneity in FL \cite{li2019fedmd,lin2020ensemble,zhu2021data}. 
Inspired by this, we explore how to introduce it into FedRL to deal with the agent heterogeneity.
Although KD has been previously applied in FedRL \cite{ryu2022model, mai2023server}, the methods used are not well-suited to address the heterogeneous FedRL problem.
In this section, we discuss some unique challenges posed by introducing KD to address the problem:
\begin{itemize}[leftmargin=10pt]
\item \textbf{Public Dataset Scarcity.} KD in FL achieves knowledge sharing among clients via public dataset.
Most FL works foucs on classification tasks in supervised learning, where training and testing datasets are available. 
However, in RL, training data is obtained through interaction with the environment.
Therefore, acquiring public dataset is the critical issue.
To tackle this issue, we generate an offline public state set as shared input data using a virtual agent.
Consistent with the core intent of GANs \cite{goodfellow2020generative}, we aim to generate synthetic data using this virtual agent.

\item \textbf{Knowledge Representation Limitation.} 
In RL, especially when using policy gradient algorithms, we need to solve tasks in continuous action spaces.
Traditional federated KD methods typically use logits as the form of knowledge representation, which is common in classification tasks but cannot be used in continuous action spaces tasks \cite{mai2023server}.
Thus, we use the action probability distribution output by the policy network as the form of knowledge representation. 
%For tasks in discrete action spaces, the action distribution can be represented as a discrete probability distribution derived from the softmax output, whereas in continuous action spaces, it is typically modeled as a continuous Gaussian distribution.
However, this also means that we need to consider the impact of distributional differences across different tasks, such as the discrete distribution derived from the softmax output and Gaussian distribution.
\end{itemize}

\subsection{FedHPD}\label{sec 4.3}
%Inspired by Ryu et al. \cite{ryu2022model}, we utilize a simulated environment model $\theta_{env}$ to obtain public dataset.
%This is reasonable in real-world because agents are typically trained in simulated environments before operating in real scenarios.
%Typically, the environment model $\theta_{env}$ can be neural network, physics engine, etc.
In practical applications, agents are typically trained in simulated environments before being deployed in real-world.
Consequently, we set up a simulated environment on the server and train a virtual agent within it.
Although optimal performance is not required, reasonable parameterization is needed to ensure stable operation in the environment.
For example, in autonomous driving tasks, the simulated environment can be based on an existing simulation platform (such as CARLA \cite{pmlr-v78-dosovitskiy17a}), where the virtual agent undergoes preliminary training to learn basic driving strategies.
Next, we conduct multiple tests with different initial state inputs to simulate the diverse scenarios that the agent may encounter.
In each test, the virtual agent generates a series of states, from which a subset of states is randomly selected to form the public state set $\mathcal{S}_p$.

\begin{figure}[t]
\centering
\includegraphics[scale=0.44]{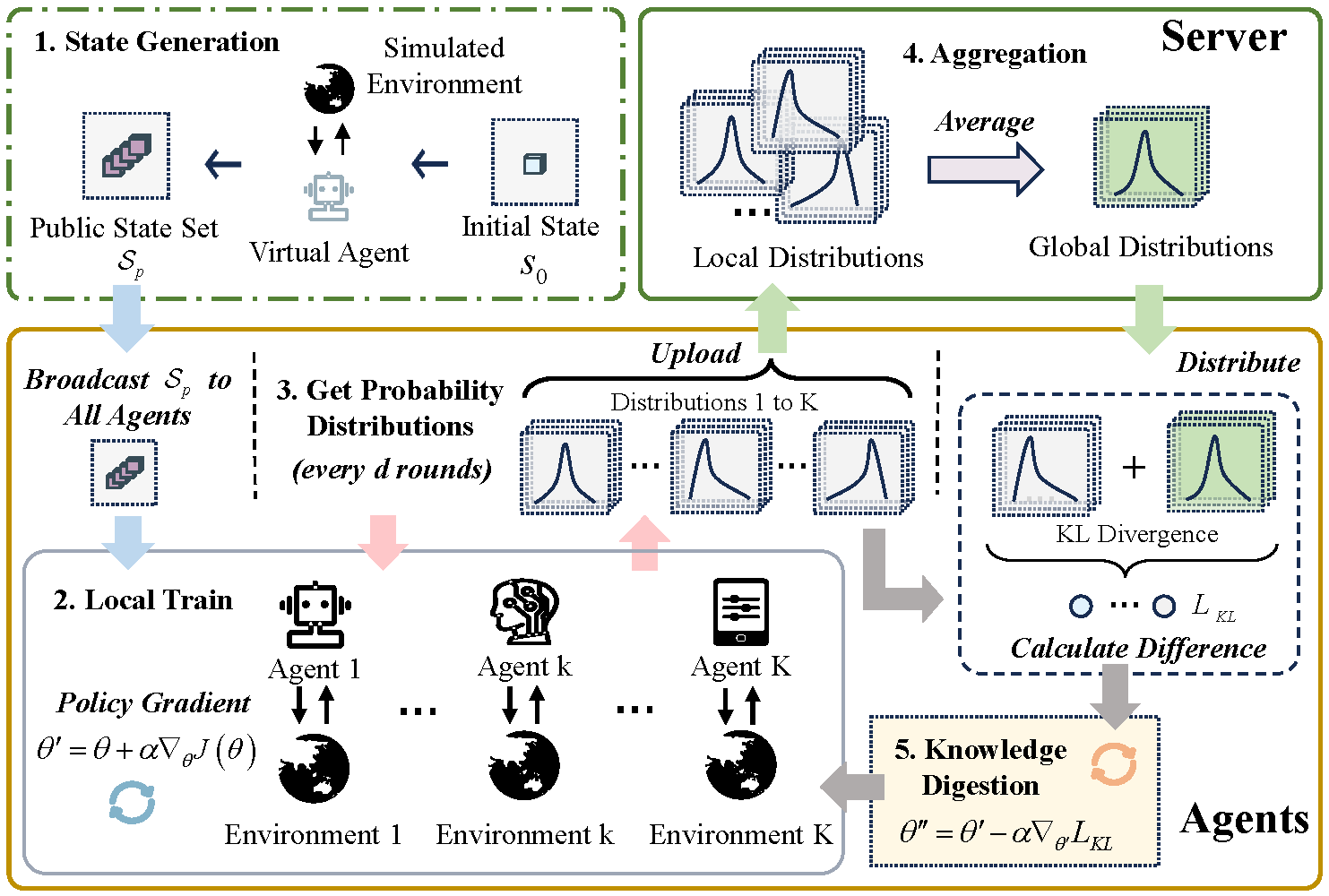}
\vspace{-0.2cm}
\caption{Illustration of FedHPD.~1. Generate public state set $S_p$ through the virtual agent;~2. Agents perform local training;~3. Get action probability distributions through public state set;~4. Knowledge aggregation to form global consensus;~5. Calculate KL divergence to execute knowledge digestion.}
\label{Graph Illustration}
\end{figure}

Notably, the synthetic state set $\mathcal{S}_p$ generated by the virtual agent is independent of local training and serves solely for KD.
Of course, for tasks like autonomous driving, robot control, and games, there is an abundance of available experience data that can directly serve as the public state set $\mathcal{S}_p$.
In conclusion, obtaining the public state set is reasonable and convenient in real-world.
For communication considerations, $\mathcal{S}_p$ is distributed to each local agent before the start of training, so that the communication only involves the upload and distribution of the knowledge.

To handle tasks both in discrete and continuous action spaces, we use action probability distribution as a bridge for communication.
Agent $k$ generates its private experience data $D_k$ by interacting with the environment for local training and REINFORCE is deployed locally in FedHPD.
%During the collaborative training, knowledge is transmitted through action probability distributions obtained under the input of the public state set $S_p$.
During the collaborative training process, knowledge (probability distribution) is distilled from the policies of heterogeneous agents using the public state set $\mathcal{S}_p$.
Subsequently, agents digest knowledge by comparing the distributions output by themselves with the global average distributions (global consensus).
In other words, local agents use public state set $\mathcal{S}_p$ and private data $D_k$ to train and improve their policy $\phi_k$, surpassing individual efforts. 
Our framework is illustrated in Fig.\ref{Graph Illustration}.

\begin{algorithm}[t]
\caption{FedHPD}\label{Algorithm 1}
\KwIn{Training Rounds $T$, Distillation Interval $d$, Initial Agent policy networks $\theta^{0}_1, \theta^{0}_2,  \ldots, \theta^{0}_{K}$, Public State Set $\mathcal{S}_p$}

\For{$i = 0, 1, \ldots, {T-1}$}
{   \tcp{Local Training}
    \For{Agent~$k\in\{1, 2, \ldots, K\}$}
    { 
        {
           $\nabla_{\theta^{i}_k} J(\theta^{i}_k) = \mathbb{E}_{\pi_{\theta^{i}_k}} \left[ \nabla_{\theta^{i}_k} \log \pi_{\theta^{i}_k}(a|s) R \right]$\\
           $\theta^{i+1}_k \leftarrow \theta^{i}_k + \alpha^{i}_{k} \nabla_{\theta^{i}_k} J(\theta^{i}_k)$
        }
    }
    \tcp{Collaborative Training}
    \If{$i+1\equiv0\pmod{d}$}
    {
                \tcp{Get Probability Distributions}
                \For{Agent~$k\in\{1, 2, \ldots, K\}$}
                {$\mathcal{P}^{i+1}_k=\pi_{\theta^{i+1}_k}(a|\mathcal{S}_p)$
                }
                Upload Probability Distributions $\mathcal{P}^{i+1}_k, k\in\{1, 2, \ldots, K\}$ to Server\\
                \tcp{Knowledge Aggregation}
                $\overline{\mathcal{P}^{i+1}}=\frac{1}{K}\sum\limits_{k=1}^{K}{\mathcal{P}^{i+1}_k}$~\tcp{Global Consensus}
                Server Broadcast $\overline{\mathcal{P}^{i+1}}$ to All Agents\\
                \tcp{Knowledge Digestion}
        \For{Agent~$k\in\{1, 2, \ldots, K\}$}
            {$\mathcal{L}^l_k=D_{KL}(\mathcal{P}^{i+1}_k~\|~\overline{\mathcal{P}^{i+1}})$\\ 
            $\theta^{i+1}_k \leftarrow \theta^{i+1}_k-\alpha^{i}_{k}\nabla_{\theta^{i+1}_k}\mathcal{L}^l_k$
            }
    }
}
\end{algorithm}

The pseudocode for FedHPD is described in Algorithm \ref{Algorithm 1}, which consists of two training stages:

(1)~Local Training (line 1-4): In each training round, all local agents must interact with the environment, using the generated data $D_k$ to update their own policy parameters $\theta^{i}_k$, resulting in $\theta^{i+1}_k$. Note that no collaboration is required at this stage.

(2)~Collaborative Training (line 5-13): In the collaborative training phase (conducted every $d$ training rounds), agents share their knowledge based on the output distributions under the public state set $\mathcal{S}_p$. The detailed steps are as follows: 
\begin{itemize}[leftmargin=10pt]
\item Get Probability Distributions (line 5-7): Each local agent obtains the action distributions ${\mathcal{P}^{i+1}_k}$ under state set $\mathcal{S}_p$ based on its own parameterized policy $\pi(\theta^{i+1}_k)$, and uploads ${\mathcal{P}^{i+1}_k}$ to server; 
\item Knowledge Aggregation (line 8-10): Server aggregates the uploaded probability distributions to obtain global consensus $\overline{\mathcal{P}^{i+1}}$ and sends $\overline{\mathcal{P}^{i+1}}$ to  local agents;
\item Knowledge Digestion (line 11-13): Local agents calculate the KL divergence between their predicted distributions ${\mathcal{P}^{i+1}_k}$ and global distributions $\overline{\mathcal{P}^{i+1}}$ to update their policy parameters $\theta^{i+1}_k$.
\end{itemize}

It is noteworthy that, to reduce communication overhead, we set a distillation interval $d$, where agents engage in collaborative training after every $d$ rounds of local training.
A larger distillation interval implies fewer federated rounds are required. 
In summary, FedHPD achieves collaborative training of heterogeneous agents through policy distillation, enabling effective policy updates without additional sampling.

\section{Theoretical Analysis}\label{sec.5 Theoretical Analysis}
In this section, we delve into the convergence of policy optimization in the context of KD.
Specifically, we focus on two questions: 
\begin{itemize}[leftmargin=10pt]
    \item Does the introduction of KD affect the convergence of the original policy gradient algorithm ? (Theorem \ref{theorem: convergence of PG with KD})
    \item Can FedHPD achieve fast convergence ? (Corollary \ref{fast convergence FedHPD})
\end{itemize}

%Due to the model heterogeneity, we analyze the convergence of individual agents. 
Let $J(\pi_{\theta_k})$ be the objective function for the policy optimization of the $k$-th agent.
Given the introduction of KD, the objective function can be reformulated as:
\begin{equation}\label{eq: J'}
J'(\theta_k) = J(\theta_k) - \lambda~ D_{\text{KL}}(\pi_{\theta_k}(a|\mathcal{S}_p) || \pi_{global}(a|\mathcal{S}_p)),
\end{equation}
where $\lambda>0$ is the regularization coefficient. $\pi_{\theta_k}(a|\mathcal{S}_p)$ is the action probability distributions output by the policy $\pi_{\theta_k}$ over the state set $\mathcal{S}_p$ and global action probability distributions $\pi_{global}(a|\mathcal{S}_p) = \frac{1}{K}\sum_{k=1}^K \pi_{\theta_k}(a|\mathcal{S}_p)$.

Eq.~(\ref{eq: J'}) redefines the objective function, and by introducing the KL divergence term, we can reduce the difference between the local policies through multiple updates.
Interestingly, the objective function of $J'(\theta_k)$ is similar to PPO with KL penalty \cite{schulman2017proximal}. 
The difference is that in our research, we aim to integrate global consensus through the KL divergence term.

Here, we provide some assumptions required for theoretical analysis, all of which are common in the literature.
\begin{assumption}[Policy parameterization regularity]\label{policy derivate}
Let $\pi_{\theta}(a|s)$ be the policy of an agent at state $s$. For every $s, a\in\mathcal{S}\times\mathcal{A}$, there exists constants $G, M>0$ such that the log-density of the policy function satisfies:
\begin{equation}\|\nabla_{\theta}\log\pi_{\theta}(a|s)\|\leq G,\quad\left\|\nabla_{\theta}^{2}\log\pi_{\theta}(a|s)\right\|\leq M.\end{equation}
\end{assumption}
Assumption \ref{policy derivate} is aimed to ensure that the changes in the policy within the parameter space are controllable, providing a foundation for the smoothness assumption of the objective function \cite{pmlr-v48-allen-zhua16,pmlr-v48-reddi16}.
\begin{proposition}\label{L-smooth of J(theta)}
Under Assumption \ref{policy derivate}, $J({\theta})$ is $L$-smooth. Therefore, for all $\theta, \theta' \in \mathbb{R}^d$, there exist a constant $L_J > 0$ satisfies:
\begin{equation}
\|\nabla_{\theta}J({\theta})-\nabla_{\theta'}J({\theta'})\|\leq L_J\|\theta-\theta'\|.
\end{equation}
\end{proposition}

Proposition \ref{L-smooth of J(theta)} is important for demonstrating convergence and its proof can be found in Xu et al.
\cite{xu20a_uai}. 
Additionally, Assumption \ref{policy derivate} and Proposition \ref{L-smooth of J(theta)} are widly used in policy gradient methods \cite{pmlr-v80-papini18a,fan2021fault,pmlr-v151-yuan22a,pmlr-v202-fatkhullin23a}.

\begin{assumption}[Softmax Function Smoothness]\label{Softmax Function Smoothness}
Let $\pi_{\theta}(a_i|\mathcal{S}_p)$ denote the probability distribution obtained from the softmax function:
\begin{equation}
\pi_{\theta}(a_i|\mathcal{S}_p)=\frac{e^{\theta^Tf(a_i,\mathcal{S}_p)}}{\sum_je^{\theta^Tf(a_j,\mathcal{S}_p)}}.\notag    
\end{equation}
Assume that the softmax function is $L$-smooth with respect to $\theta$. Specifically, this means:

(i) The first-order derivatives of $\pi_{\theta}(a_i|\mathcal{S}_p)$ with respect to $\theta$ are continuous:
\begin{equation}
\nabla_{\theta} \pi_{\theta}(a_i|\mathcal{S}_p) = \pi_{\theta}(a_i|\mathcal{S}_p) \left( \mathbf{e}_i - \pi_{\theta}(\cdot|\mathcal{S}_p) \right),\notag    
\end{equation}
where $\mathbf{e}_i$ is the unit vector corresponding to $a_i$.

(ii) The second-order derivatives of $\pi_{\theta}(a_i|\mathcal{S}_p)$ with respect to $\theta$ are continuous:
\begin{equation}
\nabla_{\theta}^2 \pi_{\theta}(a_i|\mathcal{S}_p) = \text{diag}(\pi_{\theta}(\cdot|\mathcal{S}_p)) - \pi_{\theta}(\cdot|\mathcal{S}_p) \pi_{\theta}(\cdot|\mathcal{S}_p)^\top.\notag    
\end{equation}
\end{assumption}

\begin{assumption}[Gaussian Distribution Smoothness]\label{Gaussian Distribution Smoothness}
Let $\pi_{\theta}(a|\mathcal{S}_p)$ denote the probability density function of a Gaussian distribution:
\begin{equation}
\pi_{\theta}(a|\mathcal{S}_p) = \frac{1}{\sqrt{2 \pi \sigma^2}} e^{\left(-\frac{(a - \mu)^2}{2 \sigma^2}\right)}, \notag   
\end{equation}
where $\mu = \mu(\theta)$ and $\sigma^2 = \sigma^2(\theta)$.

Assume that $\mu(\theta)$ and $\sigma^2(\theta)$ are $L$-smooth functions of $\theta$. Specifically, this means:

(i) The first-order derivatives of $\mu(\theta)$ and $\sigma^2(\theta)$ with respect to $\theta$ are continuous: $\nabla_{\theta} \mu(\theta) \text{ and } \nabla_{\theta} \sigma^2(\theta) \text{ are continuous}.$

(ii) The second-order derivatives of $\mu(\theta)$ and $\sigma^2(\theta)$ with respect to $\theta$ are continuous: $\nabla_{\theta}^2 \mu(\theta) \text{ and } \nabla_{\theta}^2 \sigma^2(\theta) \text{ are continuous}.$
\end{assumption}

Assumptions 2 and 3 establish the foundation for the $L$-smoothness of the KL divergence term, and it is reasonable to assume smoothness for distributions in machine learning \cite{baxter2001infinite,silver2014deterministic,NIPS2016_b3ba8f1b,sutton2018reinforcement}.

\begin{lemma}[KL Divergence Smoothness]\label{D_KL smooth}
Under Assumptions \ref{Softmax Function Smoothness}, \ref{Gaussian Distribution Smoothness}, the KL divergence term $D_{\text{KL}}(\pi_{\theta_k}(a|\mathcal{S}_p) || \pi_{global}(a|\mathcal{S}_p))$ is $L$-smooth. Specifically, there exists a constant $L_{KL} > 0$ such that for any two parameter sets $\theta_k$ and $\theta_k'$,  the following inequality holds:
\begin{equation}
\begin{aligned}
\|\nabla_{\theta_k}D_{\text{KL}}(\pi_{\theta_k} \| \pi_{global}) &- \nabla_{\theta_k'}D_{\text{KL}}(\pi_{\theta_k'} \| \pi_{global})\|\\&\leq L_{KL} \|\theta_k - \theta_k'\|,     
\end{aligned}   
\end{equation}
where $\pi_{\theta_k}$ and $\pi_{\theta_k'}$ are probability distributions parameterized by $\theta_k$ and $\theta_k'$ respectively.
\end{lemma}

Lemma \ref{D_KL smooth} is crucial for the convergence of the FedHPD. 
It is related to the $L$-smoothness of the overall objective function $J'({\theta_k})$ after introducing KD.
Due to space constraints, the proof of Lemma \ref{D_KL smooth} is deferred to Appendix \ref{proof of lemma 1}.

\begin{theorem}[Convergence of REINFORCE with Knowledge Distillation]\label{theorem: convergence of PG with KD}
Under Assumptions \ref{policy derivate}, \ref{Softmax Function Smoothness}, \ref{Gaussian Distribution Smoothness}, if we choose $L = L_J - \lambda L_{KL} > 0$ and learning rates $\alpha^i$ satisfy Robbins-Monro conditions, the objective function $J'({\theta_k})$ can converge to a stationary point as the number of iterations increases. 
\end{theorem}

The proof of Theorem \ref{theorem: convergence of PG with KD} is deferred to Appendix \ref{proof: theorem 1}. 
Theorem \ref{theorem: convergence of PG with KD} analyzes the convergence conditions for REINFORCE with KD, but does not consider the effect of distillation interval.
Building upon Theorem \ref{theorem: convergence of PG with KD}, we incorporate the distillation interval $d$, leading to the analysis of FedHPD:

\begin{corollary}[Fast Convergence of FedHPD]\label{fast convergence FedHPD}
In the setting of Theorem \ref{theorem: convergence of PG with KD}, if we choose $\lambda=1$, FedHPD can achieve fast convergence by reducing variance, leading to:
\begin{equation}
N_{REINFORCE with KD}\leq N_{FedHPD} \leq N_{REINFORCE},
\end{equation}
where $N$ represents the sample size, and REINFORCE with KD refers to FedHPD without distillation intervals.
\end{corollary}

\begin{proof}
According to the description of FedHPD in Section \ref{sec 4.3}, we introduce knowledge distillation with a period of $d$. 

The objective function of FedHPD can be formulated as:
\begin{equation}
J''(\theta_k)=\begin{cases}J(\theta_k)~,~i~\%~d \neq 0.\\J'(\theta_k)~,~i~\%~d = 0.\end{cases}
\end{equation}

In the setting of Theorem \ref{theorem: convergence of PG with KD}, $J(\theta_k)$ will progressively approach the global optimal solution, thus ensuring convergence \cite{williams1992simple,sutton2018reinforcement}.
Theorem \ref{theorem: convergence of PG with KD} demonstrates that $J'(\theta_k)$ is $L$-smooth and converges under the learning rates which satisfy Eq. (\ref{Robbins-Monro conditions}).

Therefore, the optimization process of $J''(\theta_k)$ alternates between two sub-processes that have been proven to converge.
As long as the learning rates $\alpha^i$ satisfy Eq. (\ref{Robbins-Monro conditions}), the entire optimization process will still converge.

Next, we discuss the fast convergence of FedHPD.
For the vanilla REINFORCE, the sampled gradient for a single trajectory is:
\begin{equation}
\hat{g}(\theta)=\sum_{t}\nabla_{\theta}\log\pi_{\theta}(a_{t}|s_{t})R(\tau).
\end{equation}

The variance of the gradient can be formulated as:
\begin{align}
\mathrm{Var}&[\nabla_\theta J(\theta)]=\mathbb{E}\left[\left(\sum_t\nabla_\theta\log\pi_\theta(a_t|s_t)R(\tau)\right)^2\right] \notag\\&-\left(\mathbb{E}\left[\sum_t\nabla_\theta\log\pi_\theta(a_t|s_t)R(\tau)\right]\right)^2. 
\end{align}

To prevent the KL divergence term from altering the overall gradient direction while reducing gradient variance, $\lambda$ should be a relatively small positive number.
The variance of the new gradient can be formulated as:
\begin{align}\label{eq: variance of the new gradient}
\mathrm{Var}[\nabla_\theta J^{\prime}(\theta)]&=\mathrm{Var}\left[\nabla_\theta J(\theta)\right]+{\lambda^2}\mathrm{Var}[\nabla_\theta D_{\mathrm{KL}}(\pi_\theta||\pi_{\mathrm{global}})]\notag\\&-2\lambda\mathrm{Cov}(\nabla_\theta J(\theta),\nabla_\theta D_\mathrm{KL}(\pi_\theta||\pi_\mathrm{global})).
\end{align}

The selection on the value of $\lambda$ is deferred to Appendix \ref{Discussion lamda value}.
In FedHPD, substituting $\lambda = 1$, we have:
\begin{align}
\mathrm{Var}[\nabla_\theta J^{\prime}(\theta)]&=\mathrm{Var}\left[\nabla_\theta J(\theta)\right]+\mathrm{Var}[\nabla_\theta D_{\mathrm{KL}}(\pi_\theta||\pi_{\mathrm{global}})]\notag\\&-2\mathrm{Cov}(\nabla_\theta J(\theta),\nabla_\theta D_\mathrm{KL}(\pi_\theta||\pi_\mathrm{global})).
\end{align}

Therefore, the variance between the gradients of the two objective functions satisfies:
\begin{equation}\label{eq:var<var}
\mathrm{Var}[\nabla_\theta J'(\theta)]<\mathrm{Var} [\nabla_\theta J(\theta)],
\end{equation}
if $\mathrm{Var}[\nabla_\theta D_{\mathrm{KL}}(\pi_\theta||\pi_{\mathrm{global}})]\notag<2\mathrm{Cov}(\nabla_\theta J(\theta),\nabla_\theta D_\mathrm{KL}(\pi_\theta||\pi_\mathrm{global}))$.

The above condition can be reformulated as:
\begin{equation}\label{eq:fast covergence condition}
    \cos(\theta_{J\&KL})>\frac{1}{2}\frac{\|\nabla_\theta D_{\mathrm{KL}}(\pi_\theta||\pi_{\mathrm{global}})\|}{\|\nabla_\theta J(\theta)\|},
\end{equation}
where $\theta_{J\&KL}$ denotes the angle between $\nabla_\theta D_{\mathrm{KL}}(\pi_\theta||\pi_{\mathrm{global}})$ and $\nabla_\theta J(\theta)$.

Eq.~(\ref{eq:fast covergence condition}) indicates that the condition for variance reduction is that enhancing rewards and reducing the discrepancy with the global consensus must align to some extent.
From the standpoint of this, a lower value of $\frac{\|\nabla_\theta D_{\mathrm{KL}}(\pi_\theta||\pi_{\mathrm{global}})\|}{\|\nabla_\theta J(\theta)\|}$ makes variance reduction easier to achieve.
Actually, it typically tends to lower values to allow effective learning while maintaining stability \cite{pmlr-v37-schulman15,schulman2017proximal,NEURIPS2019_Divergence-Augmented,pmlr-v124-touati20a}, especially in the early stages of the training.
%Therefore, the key to achieving fast convergence lies in whether the difference between policy updates and collaborative updates falls within the limits described by Eq.~(\ref{eq:fast covergence condition}).
A smaller distillation interval results in smaller differences between  policies of heterogeneous agents, thus making Eq.~(\ref{eq:fast covergence condition}) easier to satisfy.

%Actually, the variance of the KL divergence term is usually small because the KL divergence serves as a smoothness constraint on the policy \cite{pmlr-v37-schulman15,schulman2017proximal,NEURIPS2019_Divergence-Augmented,pmlr-v124-touati20a}.
%The negative covariance term can suppress excessive fluctuations in policy updates, reducing drastic changes in the gradient and making the training process more stable.

Combined with Eq. (\ref{eq:var<var}), for FedHPD, we have:
\begin{equation}\label{eq:var<var<var}
\mathrm{Var}[\nabla_\theta J'(\theta)]\leq\mathrm{Var} [\nabla_\theta J''(\theta)]\leq\mathrm{Var} [\nabla_\theta J(\theta)],
\end{equation}
where the equality is obtained for $d=1$ and $d=\infty$.

The reduction in variance implies that, for a given precision $\epsilon$, fewer samples are required to achieve the same level of accuracy. 
According to concentration inequalities (e.g., Chebyshev's inequality):
\begin{equation}
\mathbb{P}\left(|\hat{g}-\mathbb{E}[g]|\geq\epsilon\right)\leq\frac{\mathrm{Var}[\hat{g}]}{N\epsilon^2}.
\end{equation}

To ensure that this probability is below a threshold $\delta$, we require:
\begin{equation}
N\geq\frac{\mathrm{Var}[\hat{g}]}{\delta\epsilon^2}.
\end{equation}

Combined with Eq. (\ref{eq:var<var<var}), the proof can be completed.
\end{proof}

\begin{remark}
Existing FedRL frameworks with theoretical convergence focus on homogeneous agents. 
We provide a preliminary theoretical analysis of the convergence conditions for heterogeneous agents in FedHPD and the fast convergence of FedHPD, offering guidance and reference for experimental validation.
\end{remark}

\section{Experiments}\label{sec.6 Experiments}
We conduct experiments to validate the effectiveness of FedHPD using three representative tasks from OpenAI Gym \cite{brockman2016openai}: Cartpole and LunarLander for discrete action spaces, and InvertedPendulum for continuous action spaces.
Our experimental design aligns with the objective functions presented in Section \ref{sec.4.1}, focusing on two critical aspects of improvement:
\begin{itemize}[leftmargin=10pt]
    \item \textbf{System-level Performance:} We assess whether FedHPD achieves higher average rewards across all agents compared with independent training.
    \item \textbf{Individual Agent Performance:} We evaluate whether each agent in the federated system demonstrates improved performance relative to its independently trained counterpart.
\end{itemize}

\subsection{Experimental Settings} 
In our experiments, we set up 10 heterogeneous agents, each trained locally using the vanilla on-policy REINFORCE algorithm. 
%All agents are heterogeneous, with unknown internal details. 
Each agent differs in policy networks (model architecture and activation functions) and training configurations (learning rates). 
Tables \ref{tab:Hyperparameters}, \ref{tab:Experiment Cartpole Configurations}, \ref{tab:Experiment LunarLander Configurations}, and \ref{tab:Experiment InvertedPendulum Configurations} in Appendix~\ref{Appendix:Configurations} summarizes the variations among the agents for the three tasks, respectively.

To examine the effect of the distillation interval $d$ on FedHPD during the collaborative training phase, we run FedHPD using various $d$ values ($d = 5, 10, 20$) in validating the performance improvement.
In all experimental results, we report the average rewards under different random seeds. Each experiment is independently run with five different random seeds (20, 25, 30, 35, 40)\footnote{Our code is available at \href{https://github.com/WinzhengKong/FedHPD}{https://github.com/WinzhengKong/FedHPD}.}.

%
% seed 42 is always a good number to use : )
%

\subsection{System Performance Improvement}
First, we evaluate the system-level performance improvement of FedHPD, as described in Eq.~(\ref{System object}). For the entire system, we focus on the average performance of all agents, comparing the system performance under independent local training (NoFed) and FedHPD with different distillation intervals ($d = 5, 10, 20$).

\begin{table}[t]
	\caption{System's Average Rewards under Different Tasks.}
	\label{tab:average returns of the system}
    \vspace{-0.2cm}
    \centering
	\begin{tabular}{cccccc}\toprule
		\textit{Algorithms} & \textit{Cartpole} & \textit{LunarLander} & \textit{InvertedPendulum} \\ \midrule
	    NoFed & 243.19 & 37.86 & 153.91 \\
        FedHPD(d=5) & 338.47 & 91.95 & 327.01 \\
        FedHPD(d=10) & 319.66 & 83.78 & 293.27 \\
        FedHPD(d=20) & 286.67 & 61.39 & 266.35 \\
        \bottomrule
	\end{tabular}
\end{table}

\begin{figure*}[bhtp]	
	\begin{minipage}{0.32\linewidth}	\centerline{\includegraphics[width=\textwidth]{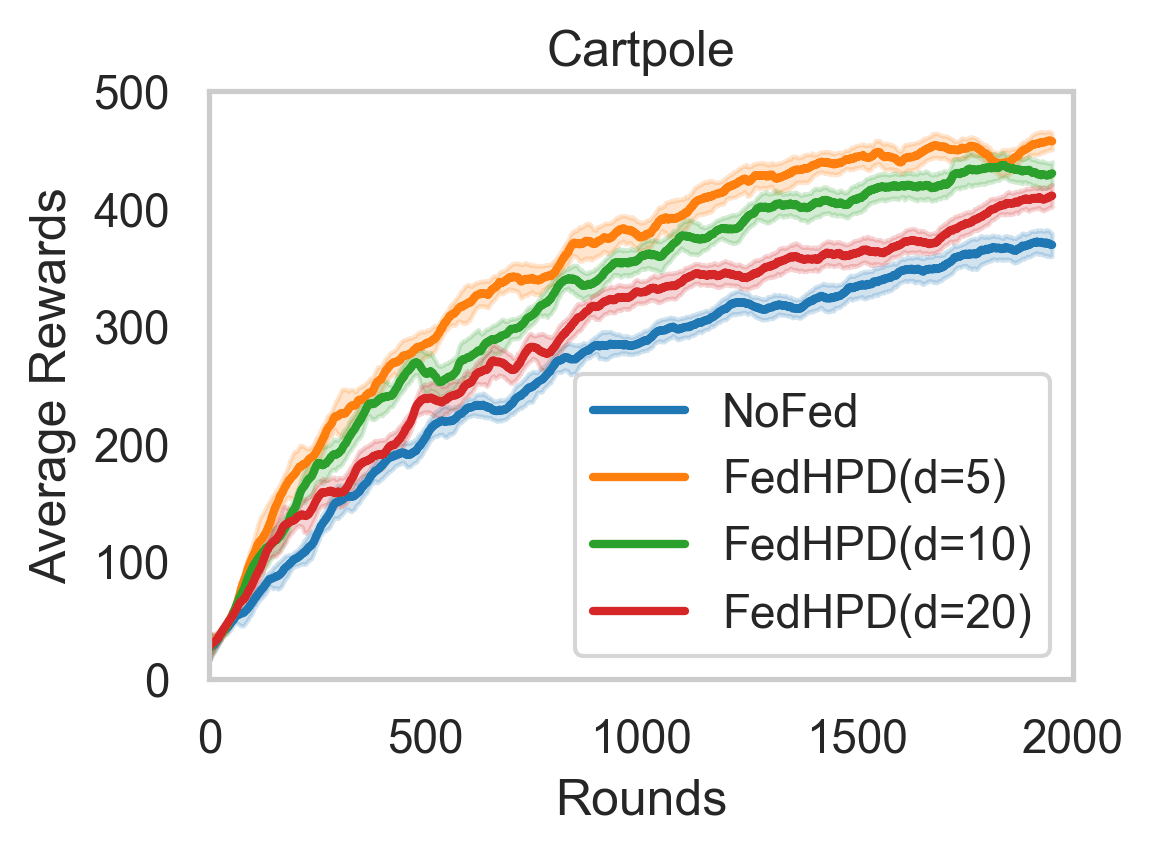}}
          % 加入对这列的图片说明
	\end{minipage}
	\begin{minipage}{0.32\linewidth}		\centerline{\includegraphics[width=\textwidth]{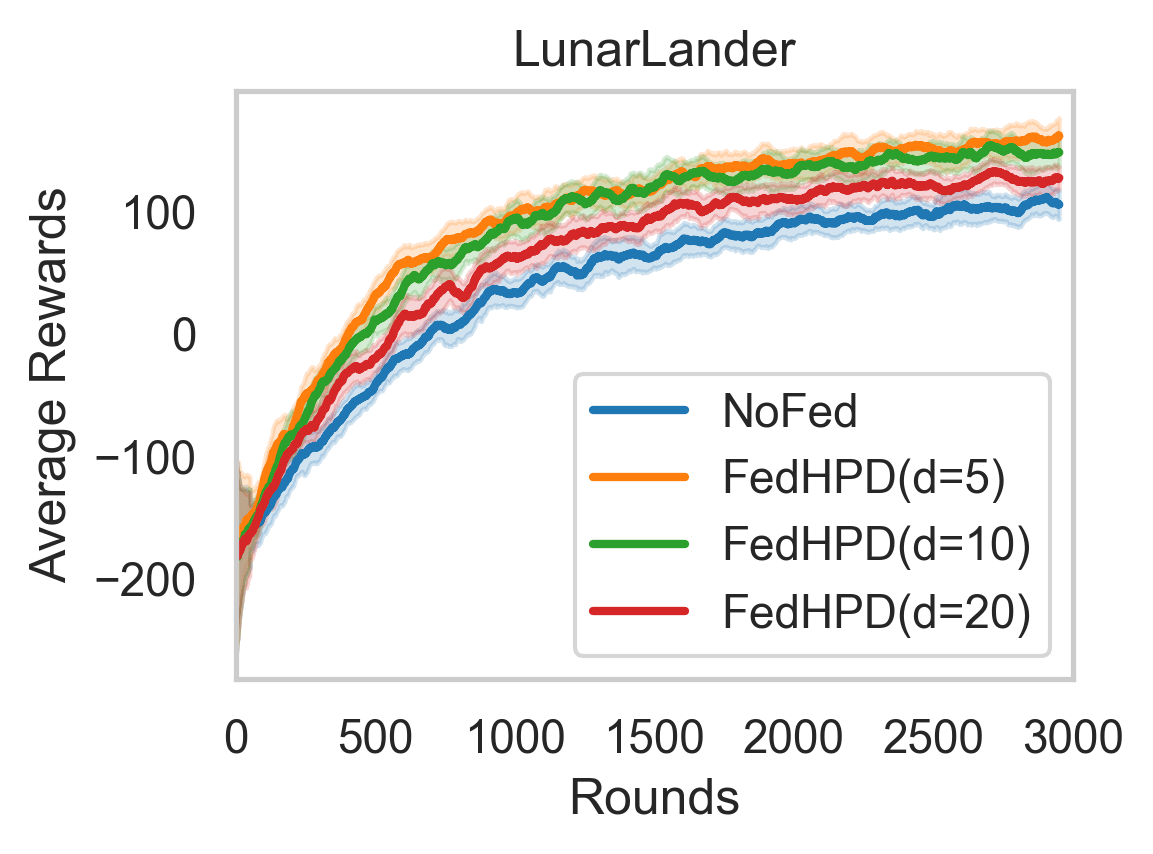}}
	\end{minipage}
	\begin{minipage}{0.32\linewidth}	\centerline{\includegraphics[width=\textwidth]{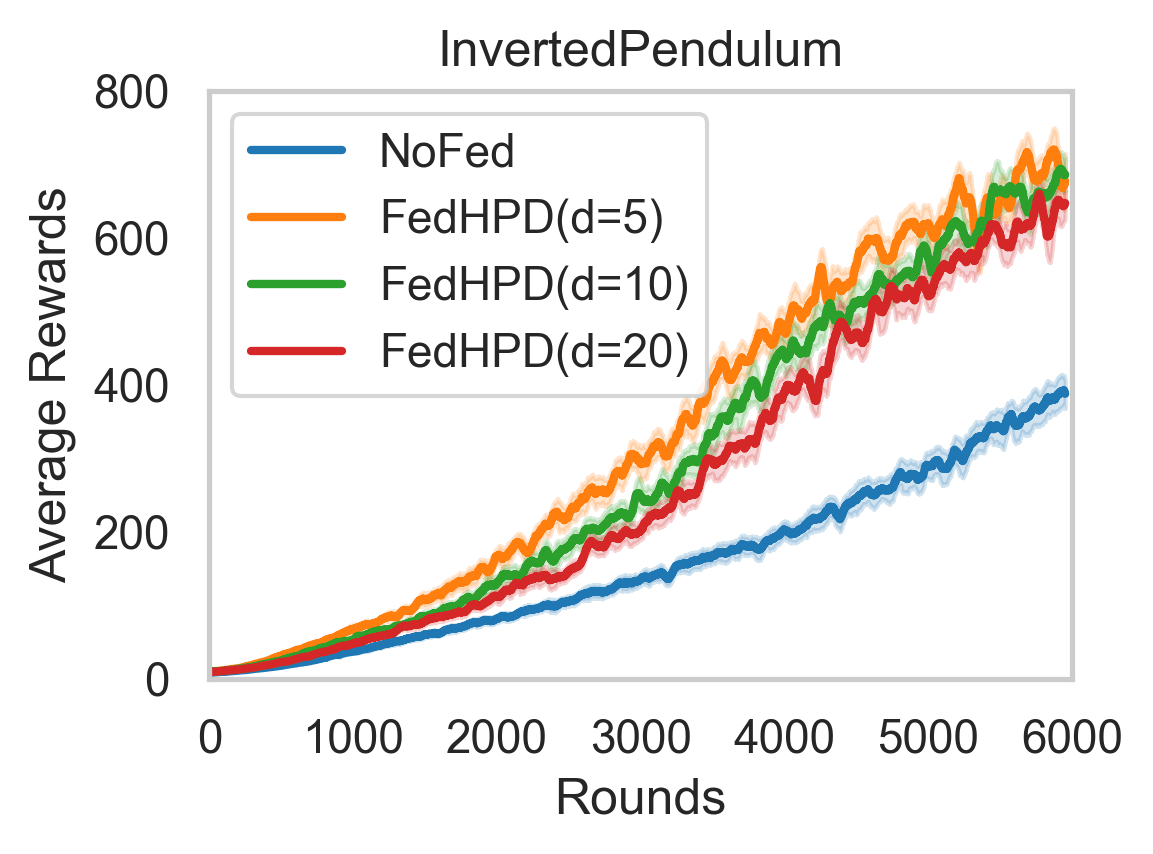}}
	\end{minipage}
    \vspace{-0.2cm}
	\caption{Comparisons of system performance under different distillation intervals ($d$ = 5, 10, 20).}
	\label{system}
\end{figure*}
\begin{figure*}[tp]	
\begin{minipage}{1\linewidth}
	\centerline{\includegraphics[width=0.96\textwidth]{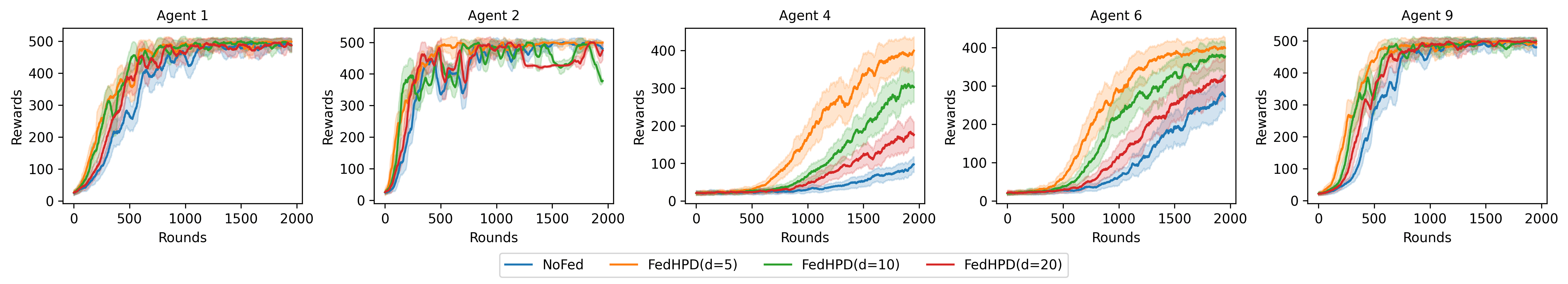}}	\centerline{\includegraphics[width=0.96\textwidth]{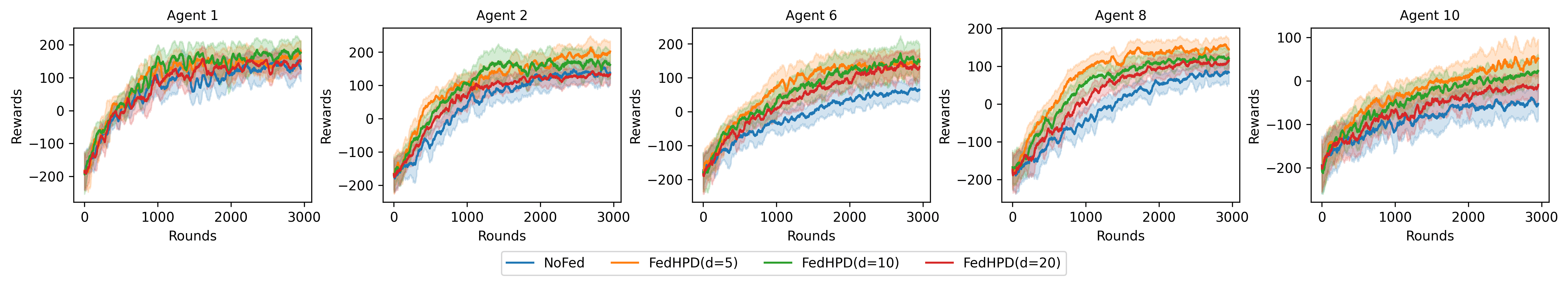}}
    \centerline{\includegraphics[width=0.96\textwidth]{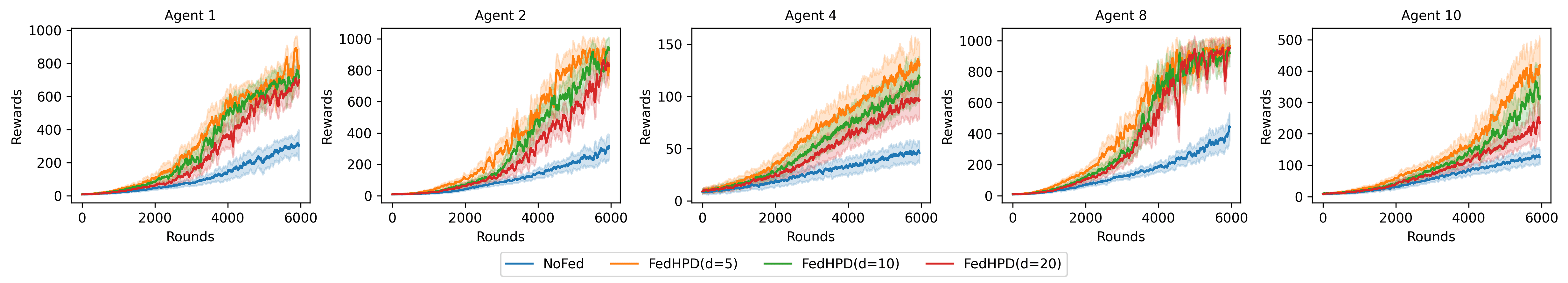}}
	\end{minipage}
    \vspace{-0.2cm}
	\caption{Comparisons of selected individual performance under different distillation intervals ($d$ = 5, 10, 20).}
	\label{individual agent partial}
\end{figure*} 

The average rewards of the system during the entire training process are presented in Table \ref{tab:average returns of the system}.
It is evident that, compared to NoFed, FedHPD achieves superior results across all tasks.
The system performance comparison of the three tasks is shown in Fig. \ref{system}. 
For the Cartpole task, the system typically needs around 1200 local training rounds to achieve a reward above 300 under NoFed, whereas with FedHPD, it requires only about 700 training rounds on average to reach the same reward. For the InvertedPendulum task, with a limited number of training rounds, NoFed can only obtain a reward of approximately 400, but with FedHPD (regardless of the $d$ value), the system consistently achieves rewards exceeding 600.

Additionally, it can be observed that the value of $d$ has less impact on overall system performance in LunarLander and InvertedPendulum compared to Cartpole. We hypothesize that in simpler tasks, where policies converge faster, the distillation frequency significantly affects the speed and stability of policy optimization. In contrast, in more complex tasks, the effect of adjusting the distillation interval on performance is overshadowed by the complexity of the task and the randomness of the environment, making it less pronounced.

\subsection{Individual Performance Improvement}
Next, we evaluate the improvement effect of FedHPD on individual agents, corresponding to Eq.~(\ref{Individual object}). We aim for FedHPD to enhance not only the system's average performance but also the performance of each individual agent in the system.

%In this section, we compare the reward trajectories for each agent under independent training (NoFed) and collaborative training with FedHPD across different tasks. 
Due to space constraints, Fig. \ref{individual agent partial} presents comparisons of selected agents in the Cartpole, LunarLander, and InvertedPendulum tasks from top to bottom, with the complete comparisons available in Appendix \ref{Appendix C.1}.
We find that FedHPD enhances the learning efficiency of the majority of agents and at least maintains the original learning efficiency. In Cartpole, Agent 1 and Agent 9 show performance improvements with FedHPD ($d$ = 5, 10, 20) compared to NoFed, requiring on average 300 fewer training rounds to reach convergence. In LunarLander, FedHPD enables most agents to achieve positive rewards more quickly and converge to better strategies. However, performance may slightly decline with larger $d$ values (e.g., Agents 2 and 4), analogous to the situation observed in Cartpole (e.g., Agents 2 and 7). This is attributed to suboptimal parameterizations or training configurations for these agents. In InvertedPendulum, FedHPD achieves significant performance improvements across all agents, greatly enhancing sample efficiency.

Interestingly, FedHPD shows a more pronounced improvement for agents that performed poorly under independent training, and as $d$ decreases, the training efficiency of these agents improves notably. For detailed examples, refer to Agents 4, 6, and 10 in Cartpole; Agents 6, 8, and 10 in LunarLander; and Agents 2, 4, and 10 in InvertedPendulum. This indirectly suggests that the introduction of KD can improve training efficiency, as global consensus helps guide suboptimal policies towards better ones.

\subsection{Comparison with Related Work}
%To the best of our knowledge, we are the first to introduce KD to address the federated heterogeneous policy gradient.
According to the comparison in Table \ref{tab:comparison-frl-methods}, most of the existing methods are incapable of addressing heterogeneous FedRL.
Both FedHQL and DPA-FedRL take into account the issue of agent heterogeneity,
but since FedHQL is a value-based method specifically designed for Q-learning, we select the most representative and relevant DPA-FedRL \cite{10609408} for comparison.
For experimental fairness, we adjust DPA-FedRL's server-side state query process to $d$ = 1, representing continuous queries without intervals.

Fig. \ref{pic:System baseline} presents system performance comparisons between DPA-FedRL and FedHPD ($d$ = 5, 10, 20) in Cartpole and LunarLander, while the individual comparisons are deferred to Appendix \ref{Appendix: FedHPD baselines}.
FedHPD can achieve better performance than DPA-FedRL under different distillation intervals, effectively improving system performance and individual sample efficiency.
Moreover, while DPA-FedRL can accelerate early-stage convergence, it may also lead to suboptimal policies, indicating that it does not effectively resolve the information inconsistency caused by agent heterogeneity.
\begin{figure}[htbp]
    \begin{minipage}{0.49\linewidth}	\centerline{\includegraphics[width=\textwidth]{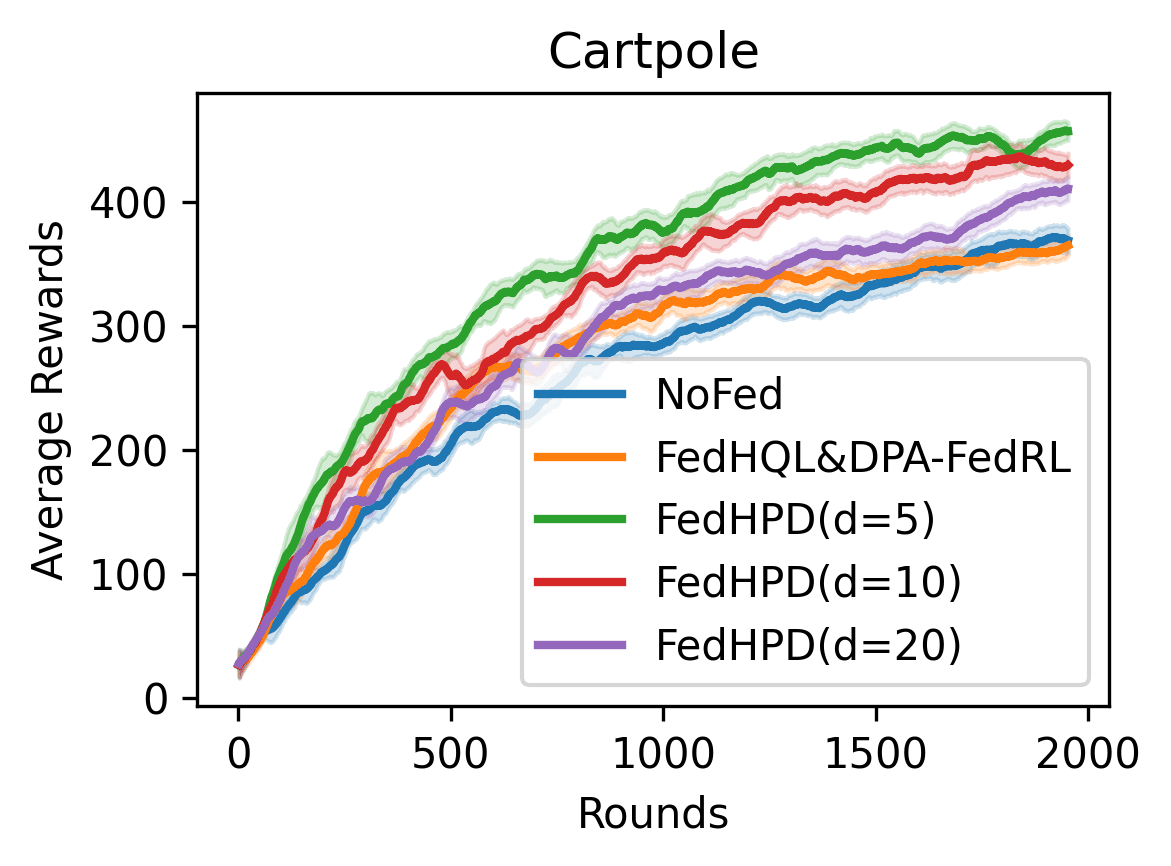}}
          % 加入对这列的图片说明
    \end{minipage}
    \begin{minipage}{0.49\linewidth}		\centerline{\includegraphics[width=\textwidth]{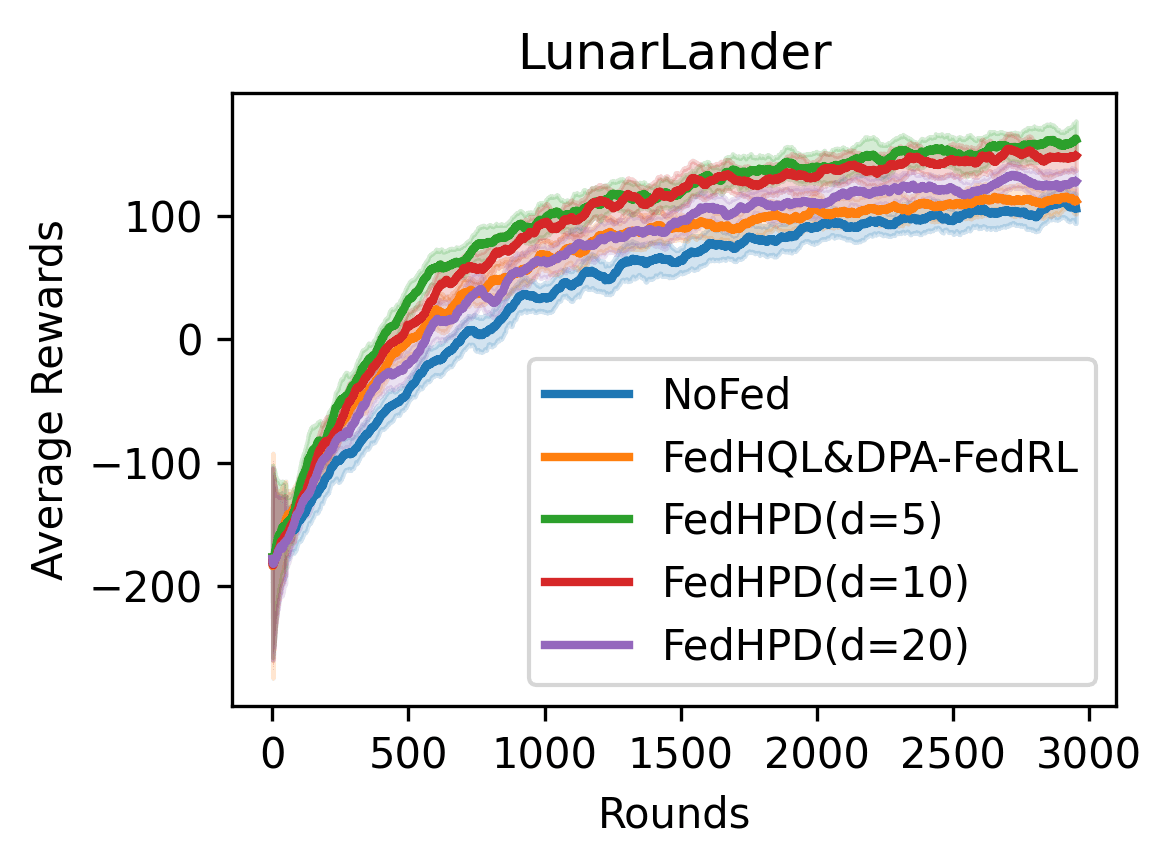}}
    \end{minipage}
\caption{Comparisons of system performance between DPA-FedRL and FedHPD ($d$ = 5, 10, 20).}
\label{pic:System baseline}
\end{figure}

\subsection{Evaluation on the Distillation Interval $d$}
% Analysis of Fig. \ref{system} and Fig. \ref{individual agent partial} reveals that increasing the distillation interval (i.e., increasing the d value) reduces system performance, but the degree of performance reduction varies for different individuals.
For the entire system, we aim to achieve the best possible performance while maintaining low communication overhead, but for some individuals, the system achieves this balance at the expense of their interests.
Therefore, in this section, we conduct a more specific evaluation on the distillation interval $d$.
\begin{figure}[ht]
    \begin{minipage}{0.49\linewidth}	\centerline{\includegraphics[width=\textwidth]{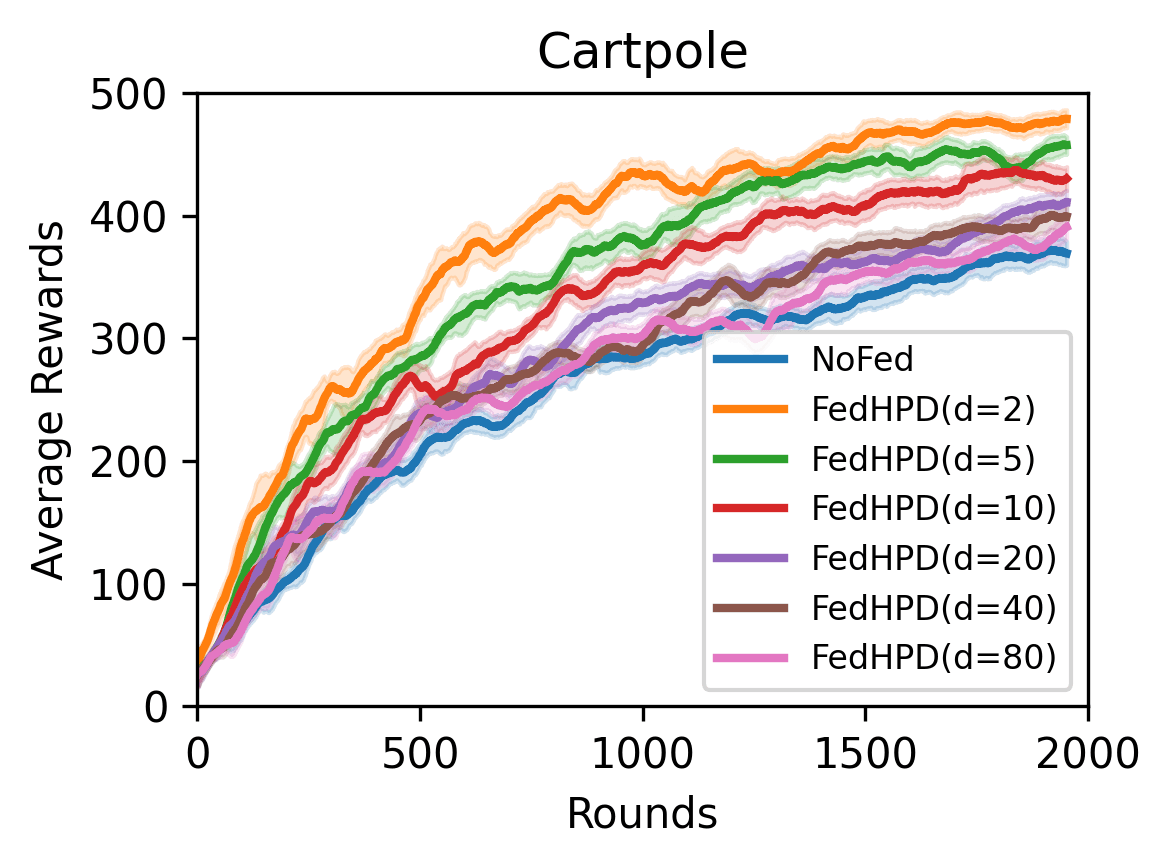}}
          % 加入对这列的图片说明
    \end{minipage}
    \begin{minipage}{0.49\linewidth}		\centerline{\includegraphics[width=\textwidth]{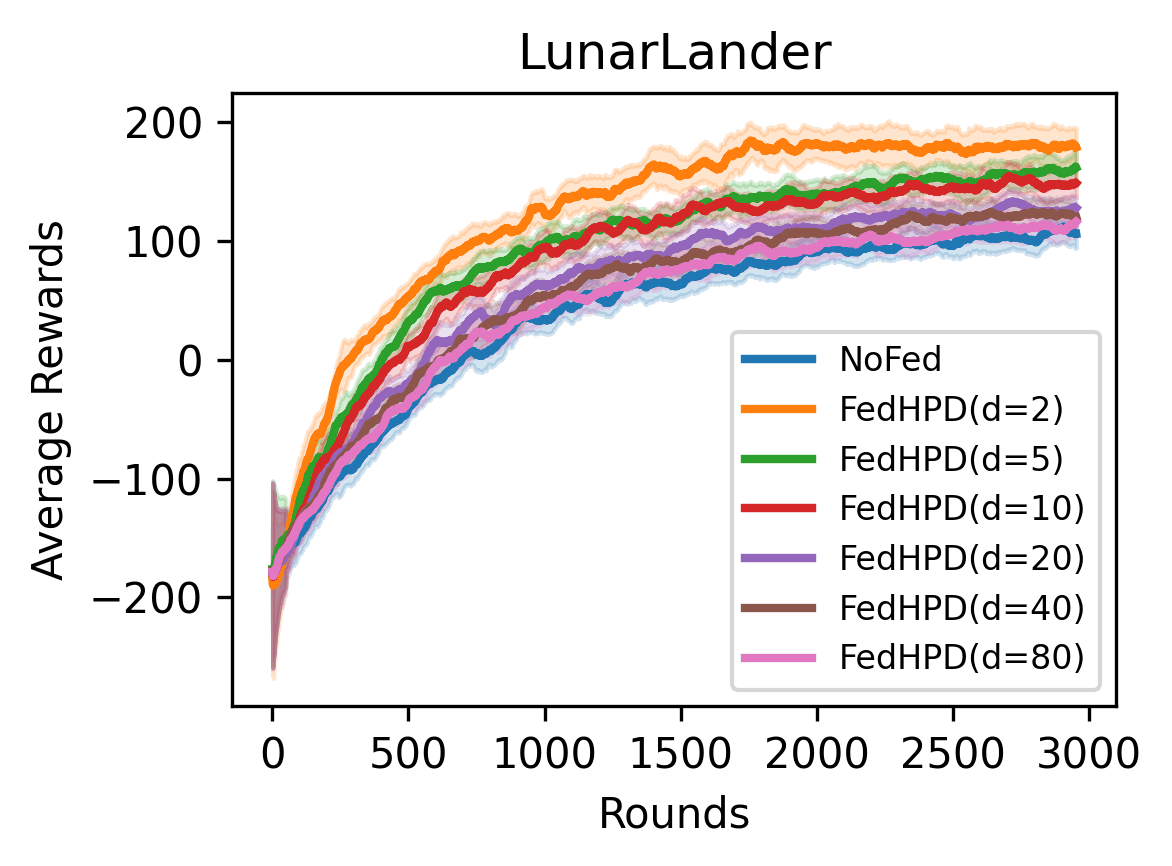}}
    \end{minipage}
    \vspace{-0.2cm}
\caption{Comparisons of system performance under different $d$ values ($d$ = 2, 5, 10, 20, 40, 80).}
\label{pic:system d explore}
\end{figure}

%Using Cartpole as a case study, 
We conduct a more in-depth analysis of how $d$ values affect the performance enhancement of FedHPD. 
We run FedHPD with six different distillation intervals ($d$ = 2, 5, 10, 20, 40, 80).
The average system performance is shown in Fig. \ref{pic:system d explore}.
As the $d$ value increases, the system's performance gradually decreases, approaching the system performance under NoFed, but still outperforming NoFed.
This validates Corollary \ref{fast convergence FedHPD}: the larger the $d$ value, the closer the sample sizes of FedHPD and REINFORCE (NoFed) become.
The analysis of individual agents is deferred to the Appendix \ref{Appendix: FedHPD under six d values}.
In light of the above analysis, we should choose an appropriate $d$ value according to practical requirements. 
If system-level optimization is desired, larger distillation interval can bring higher benefits. 
However, if the emphasis is on optimizing specific individuals, smaller interval is necessary to achieve significant improvements.

\section{Discussion}
\textbf{Diversity of Public Dataset.} In federated model distillation for supervised learning, the selection of public datasets is crucial, as it relates to the model's generalization ability \cite{ijcaisun2021}. 
When deploying model distillation in FedRL, public datasets are also utilized. Similar to unlabeled datasets, our public datasets only contain state information. The selection of public state sets in FedRL is not sensitive. We hypothesize that this is because RL is inherently an exploratory learning process, where agents can learn how to take appropriate actions in similar states through exploration, even if certain states do not appear in the public dataset. This generalization capability makes the model less sensitive to the selection of public state sets, as it can compensate for deficiencies in the public state set through exploration. Taking Cartpole as an example, we select different public state sets to run FedHPD, with specific details provided in Appendix~\ref{Appendix: FedHPD under different public state sets}.
Experiments demonstrate that FedHPD does not require rigorous selection of public dataset.
This aligns with our method of generating the offline state set $\mathcal{S}_p$, as it is only used for knowledge distillation and not for local training.

\textbf{RL algorithms for local training.}~In FedHPD, we implement the vanilla REINFORCE algorithm locally, which is an on-policy RL algorithm. However, REINFORCE may struggle with more complex tasks. In such cases, the policy gradient algorithm used for local training can be replaced with off-policy forms, using variance-reduced policy gradient methods \cite{pmlr-v80-papini18a,gargiani2022page} or Actor-Critic frameworks \cite{mnih2016asynchronous}. Given the local training does not conflict with the collaborative training, the KD process does not need to be altered with algorithm adjustment.

\textbf{Formation of global consensus.}~Global consensus is obtained by averaging the outputs of all agents on the public dataset in FedHPD, which is a common operation in KD.
However, in FedRL with agent heterogeneity, there may be significant differences in the learning progress of local agents. 
If we use average aggregation, it might lead the policy towards suboptimal solution.
In this case, if we could determine aggregation weights based on the quality of local policies, it might play an important role in extreme scenarios.

\section{Conclusion}
Most existing FedRL research assumes agents are homogeneous, which is a significant limitation in real-world scenarios. 
Departing from this limitation, this paper investigates heterogeneous FedRL, where agents are heterogeneous and their internal details are unknown.
We propose FedHPD, which effectively enhances the sampling efficiency of heterogeneous agents through multiple rounds of local training and periodic knowledge sharing.
Moreover, we theoretically analyze the fast convergence of FedHPD under standard assumptions.
Finally, we empirically validate the effectiveness of the proposed framework across various RL tasks.
%Theoretical analysis in the context of agent heterogeneity is challenging. 
This paper analyzes the convergence of FedHPD but does not provide specific sample complexity results. 
In other words, more rigorous theoretical guarantees under agent heterogeneity are worthy of investigation. 

\clearpage
%%% The next two lines define, first, the bibliography style to be 
%%% applied, and, second, the bibliography file to be used.
%\clearpage
\bibliographystyle{ACM-Reference-Format} 
\bibliography{sample}


\begin{thebibliography}{52}

%%% the Web address.

\ifx \showCODEN    \undefined \def \showCODEN     #1{\unskip}     \fi
\ifx \showDOI      \undefined \def \showDOI       #1{#1}\fi
\ifx \showISBNx    \undefined \def \showISBNx     #1{\unskip}     \fi
\ifx \showISBNxiii \undefined \def \showISBNxiii  #1{\unskip}     \fi
\ifx \showISSN     \undefined \def \showISSN      #1{\unskip}     \fi
\ifx \showLCCN     \undefined \def \showLCCN      #1{\unskip}     \fi
\ifx \shownote     \undefined \def \shownote      #1{#1}          \fi
\ifx \showarticletitle \undefined \def \showarticletitle #1{#1}   \fi
\ifx \showURL      \undefined \def \showURL       {\relax}        \fi
% invisible to TeX
\providecommand\bibfield[2]{#2}
\providecommand\bibinfo[2]{#2}
\providecommand\natexlab[1]{#1}
\providecommand\showeprint[2][]{arXiv:#2}

\bibitem[\protect\citeauthoryear{Allen-Zhu and Hazan}{Allen-Zhu and Hazan}{2016}]%
        {pmlr-v48-allen-zhua16}
\bibfield{author}{\bibinfo{person}{Zeyuan Allen-Zhu} {and} \bibinfo{person}{Elad Hazan}.} \bibinfo{year}{2016}\natexlab{}.
\newblock \showarticletitle{Variance Reduction for Faster Non-Convex Optimization}. In \bibinfo{booktitle}{\emph{ICML}}, Vol.~\bibinfo{volume}{48}. \bibinfo{pages}{699--707}.
\newblock


\bibitem[\protect\citeauthoryear{Baxter and Bartlett}{Baxter and Bartlett}{2001}]%
        {baxter2001infinite}
\bibfield{author}{\bibinfo{person}{Jonathan Baxter} {and} \bibinfo{person}{Peter~L Bartlett}.} \bibinfo{year}{2001}\natexlab{}.
\newblock \showarticletitle{Infinite-horizon policy-gradient estimation}.
\newblock \bibinfo{journal}{\emph{Journal of Artificial Intelligence Research}}  \bibinfo{volume}{15} (\bibinfo{year}{2001}), \bibinfo{pages}{319--350}.
\newblock


\bibitem[\protect\citeauthoryear{Brockman, Cheung, Pettersson, Schneider, Schulman, Tang, and Zaremba}{Brockman et~al\mbox{.}}{2016}]%
        {brockman2016openai}
\bibfield{author}{\bibinfo{person}{Greg Brockman}, \bibinfo{person}{Vicki Cheung}, \bibinfo{person}{Ludwig Pettersson}, \bibinfo{person}{Jonas Schneider}, \bibinfo{person}{John Schulman}, \bibinfo{person}{Jie Tang}, {and} \bibinfo{person}{Wojciech Zaremba}.} \bibinfo{year}{2016}\natexlab{}.
\newblock \showarticletitle{Openai gym}.
\newblock \bibinfo{journal}{\emph{arXiv preprint arXiv:1606.01540}} (\bibinfo{year}{2016}).
\newblock


\bibitem[\protect\citeauthoryear{Chai, Niu, Carrasco, Arvin, Yin, and Lennox}{Chai et~al\mbox{.}}{2022}]%
        {chai2022design}
\bibfield{author}{\bibinfo{person}{Runqi Chai}, \bibinfo{person}{Hanlin Niu}, \bibinfo{person}{Joaquin Carrasco}, \bibinfo{person}{Farshad Arvin}, \bibinfo{person}{Hujun Yin}, {and} \bibinfo{person}{Barry Lennox}.} \bibinfo{year}{2022}\natexlab{}.
\newblock \showarticletitle{Design and experimental validation of deep reinforcement learning-based fast trajectory planning and control for mobile robot in unknown environment}.
\newblock \bibinfo{journal}{\emph{IEEE TNNLS}} \bibinfo{volume}{35}, \bibinfo{number}{4} (\bibinfo{year}{2022}), \bibinfo{pages}{5778--5792}.
\newblock


\bibitem[\protect\citeauthoryear{Cuccu, Rolshoven, Vorpe, Cudr{\'e}-Mauroux, and Glasmachers}{Cuccu et~al\mbox{.}}{2022}]%
        {cuccu2022dibb}
\bibfield{author}{\bibinfo{person}{Giuseppe Cuccu}, \bibinfo{person}{Luca Rolshoven}, \bibinfo{person}{Fabien Vorpe}, \bibinfo{person}{Philippe Cudr{\'e}-Mauroux}, {and} \bibinfo{person}{Tobias Glasmachers}.} \bibinfo{year}{2022}\natexlab{}.
\newblock \showarticletitle{DiBB: distributing black-box optimization}. In \bibinfo{booktitle}{\emph{Genetic and Evolutionary Computation Conference}}. \bibinfo{pages}{341--349}.
\newblock


\bibitem[\protect\citeauthoryear{Dai, Fan, Tan, Hoang, Low, and Jaillet}{Dai et~al\mbox{.}}{2024}]%
        {DAI2024257}
\bibfield{author}{\bibinfo{person}{Zhongxiang Dai}, \bibinfo{person}{Flint~Xiaofeng Fan}, \bibinfo{person}{Cheston Tan}, \bibinfo{person}{Trong~Nghia Hoang}, \bibinfo{person}{Bryan Kian~Hsiang Low}, {and} \bibinfo{person}{Patrick Jaillet}.} \bibinfo{year}{2024}\natexlab{}.
\newblock \showarticletitle{Chapter 14 - Federated sequential decision making: Bayesian optimization, reinforcement learning, and beyond}.
\newblock In \bibinfo{booktitle}{\emph{Federated Learning}}. \bibinfo{publisher}{Academic Press}, \bibinfo{pages}{257--279}.
\newblock
\showISBNx{978-0-443-19037-7}
\urldef\tempurl%
\url{https://doi.org/10.1016/B978-0-44-319037-7.00023-5}
\showDOI{\tempurl}


\bibitem[\protect\citeauthoryear{Dosovitskiy, Ros, Codevilla, Lopez, and Koltun}{Dosovitskiy et~al\mbox{.}}{2017}]%
        {pmlr-v78-dosovitskiy17a}
\bibfield{author}{\bibinfo{person}{Alexey Dosovitskiy}, \bibinfo{person}{German Ros}, \bibinfo{person}{Felipe Codevilla}, \bibinfo{person}{Antonio Lopez}, {and} \bibinfo{person}{Vladlen Koltun}.} \bibinfo{year}{2017}\natexlab{}.
\newblock \showarticletitle{{CARLA}: {An} Open Urban Driving Simulator}. In \bibinfo{booktitle}{\emph{CoRL}}. \bibinfo{pages}{1--16}.
\newblock


\bibitem[\protect\citeauthoryear{Fan, Ma, Dai, Jing, Tan, and Low}{Fan et~al\mbox{.}}{2021}]%
        {fan2021fault}
\bibfield{author}{\bibinfo{person}{Flint~Xiaofeng Fan}, \bibinfo{person}{Yining Ma}, \bibinfo{person}{Zhongxiang Dai}, \bibinfo{person}{Wei Jing}, \bibinfo{person}{Cheston Tan}, {and} \bibinfo{person}{Bryan Kian~Hsiang Low}.} \bibinfo{year}{2021}\natexlab{}.
\newblock \showarticletitle{Fault-tolerant federated reinforcement learning with theoretical guarantee}. In \bibinfo{booktitle}{\emph{Advances in Neural Information Processing Systems}}, Vol.~\bibinfo{volume}{34}. \bibinfo{pages}{1007--1021}.
\newblock


\bibitem[\protect\citeauthoryear{Fan, Ma, Dai, Tan, and Low}{Fan et~al\mbox{.}}{2023}]%
        {FHQL2023}
\bibfield{author}{\bibinfo{person}{Flint~Xiaofeng Fan}, \bibinfo{person}{Yining Ma}, \bibinfo{person}{Zhongxiang Dai}, \bibinfo{person}{Cheston Tan}, {and} \bibinfo{person}{Bryan Kian~Hsiang Low}.} \bibinfo{year}{2023}\natexlab{}.
\newblock \showarticletitle{FedHQL: Federated Heterogeneous Q-Learning}. In \bibinfo{booktitle}{\emph{International Conference on Autonomous Agents and Multiagent Systems}}. \bibinfo{pages}{2810–2812}.
\newblock


\bibitem[\protect\citeauthoryear{Fan, Tan, Ong, Wattenhofer, and Ooi}{Fan et~al\mbox{.}}{2024}]%
        {fan2024fedrlhf}
\bibfield{author}{\bibinfo{person}{Flint~Xiaofeng Fan}, \bibinfo{person}{Cheston Tan}, \bibinfo{person}{Yew-Soon Ong}, \bibinfo{person}{Roger Wattenhofer}, {and} \bibinfo{person}{Wei-Tsang Ooi}.} \bibinfo{year}{2024}\natexlab{}.
\newblock \showarticletitle{FedRLHF: A Convergence-Guaranteed Federated Framework for Privacy-Preserving and Personalized RLHF}.
\newblock \bibinfo{journal}{\emph{arXiv preprint arXiv:2412.15538}} (\bibinfo{year}{2024}).
\newblock


\bibitem[\protect\citeauthoryear{Fatkhullin, Barakat, Kireeva, and He}{Fatkhullin et~al\mbox{.}}{2023}]%
        {pmlr-v202-fatkhullin23a}
\bibfield{author}{\bibinfo{person}{Ilyas Fatkhullin}, \bibinfo{person}{Anas Barakat}, \bibinfo{person}{Anastasia Kireeva}, {and} \bibinfo{person}{Niao He}.} \bibinfo{year}{2023}\natexlab{}.
\newblock \showarticletitle{Stochastic Policy Gradient Methods: Improved Sample Complexity for {F}isher-non-degenerate Policies}. In \bibinfo{booktitle}{\emph{ICML}}, Vol.~\bibinfo{volume}{202}. \bibinfo{pages}{9827--9869}.
\newblock


\bibitem[\protect\citeauthoryear{Feng, Liu, Huang, and Guo}{Feng et~al\mbox{.}}{2023a}]%
        {feng2023robust}
\bibfield{author}{\bibinfo{person}{Bin Feng}, \bibinfo{person}{Zhuping Liu}, \bibinfo{person}{Gang Huang}, {and} \bibinfo{person}{Chuangxin Guo}.} \bibinfo{year}{2023}\natexlab{a}.
\newblock \showarticletitle{Robust federated deep reinforcement learning for optimal control in multiple virtual power plants with electric vehicles}.
\newblock \bibinfo{journal}{\emph{Applied Energy}}  \bibinfo{volume}{349} (\bibinfo{year}{2023}), \bibinfo{pages}{121615}.
\newblock


\bibitem[\protect\citeauthoryear{Feng, Sun, Yan, Zhu, Zou, Shen, and Liu}{Feng et~al\mbox{.}}{2023b}]%
        {feng2023dense}
\bibfield{author}{\bibinfo{person}{Shuo Feng}, \bibinfo{person}{Haowei Sun}, \bibinfo{person}{Xintao Yan}, \bibinfo{person}{Haojie Zhu}, \bibinfo{person}{Zhengxia Zou}, \bibinfo{person}{Shengyin Shen}, {and} \bibinfo{person}{Henry~X Liu}.} \bibinfo{year}{2023}\natexlab{b}.
\newblock \showarticletitle{Dense reinforcement learning for safety validation of autonomous vehicles}.
\newblock \bibinfo{journal}{\emph{Nature}} \bibinfo{volume}{615}, \bibinfo{number}{7953} (\bibinfo{year}{2023}), \bibinfo{pages}{620--627}.
\newblock


\bibitem[\protect\citeauthoryear{Gargiani, Zanelli, Martinelli, Summers, and Lygeros}{Gargiani et~al\mbox{.}}{2022}]%
        {gargiani2022page}
\bibfield{author}{\bibinfo{person}{Matilde Gargiani}, \bibinfo{person}{Andrea Zanelli}, \bibinfo{person}{Andrea Martinelli}, \bibinfo{person}{Tyler Summers}, {and} \bibinfo{person}{John Lygeros}.} \bibinfo{year}{2022}\natexlab{}.
\newblock \showarticletitle{PAGE-PG: A simple and loopless variance-reduced policy gradient method with probabilistic gradient estimation}. In \bibinfo{booktitle}{\emph{ICML}}, Vol.~\bibinfo{volume}{162}. \bibinfo{pages}{7223--7240}.
\newblock


\bibitem[\protect\citeauthoryear{Goodfellow, Pouget-Abadie, Mirza, Xu, Warde-Farley, Ozair, Courville, and Bengio}{Goodfellow et~al\mbox{.}}{2020}]%
        {goodfellow2020generative}
\bibfield{author}{\bibinfo{person}{Ian Goodfellow}, \bibinfo{person}{Jean Pouget-Abadie}, \bibinfo{person}{Mehdi Mirza}, \bibinfo{person}{Bing Xu}, \bibinfo{person}{David Warde-Farley}, \bibinfo{person}{Sherjil Ozair}, \bibinfo{person}{Aaron Courville}, {and} \bibinfo{person}{Yoshua Bengio}.} \bibinfo{year}{2020}\natexlab{}.
\newblock \showarticletitle{Generative adversarial networks}.
\newblock \bibinfo{journal}{\emph{Commun. ACM}} \bibinfo{volume}{63}, \bibinfo{number}{11} (\bibinfo{year}{2020}), \bibinfo{pages}{139--144}.
\newblock


\bibitem[\protect\citeauthoryear{Hinton, Vinyals, and Dean}{Hinton et~al\mbox{.}}{2015}]%
        {hinton2015distilling}
\bibfield{author}{\bibinfo{person}{Geoffrey Hinton}, \bibinfo{person}{Oriol Vinyals}, {and} \bibinfo{person}{Jeff Dean}.} \bibinfo{year}{2015}\natexlab{}.
\newblock \showarticletitle{Distilling the knowledge in a neural network}.
\newblock \bibinfo{journal}{\emph{arXiv preprint arXiv:1503.02531}} (\bibinfo{year}{2015}).
\newblock


\bibitem[\protect\citeauthoryear{Jin, Feng, and Yu}{Jin et~al\mbox{.}}{2024}]%
        {10609408}
\bibfield{author}{\bibinfo{person}{Chenying Jin}, \bibinfo{person}{Xiang Feng}, {and} \bibinfo{person}{Huiqun Yu}.} \bibinfo{year}{2024}\natexlab{}.
\newblock \showarticletitle{Embracing Multiheterogeneity and Privacy Security Simultaneously: A Dynamic Privacy-Aware Federated Reinforcement Learning Approach}.
\newblock \bibinfo{journal}{\emph{IEEE TNNLS}} (\bibinfo{year}{2024}), \bibinfo{pages}{1--15}.
\newblock


\bibitem[\protect\citeauthoryear{Jin, Peng, Yang, Wang, and Zhang}{Jin et~al\mbox{.}}{2022}]%
        {jin2022federated}
\bibfield{author}{\bibinfo{person}{Hao Jin}, \bibinfo{person}{Yang Peng}, \bibinfo{person}{Wenhao Yang}, \bibinfo{person}{Shusen Wang}, {and} \bibinfo{person}{Zhihua Zhang}.} \bibinfo{year}{2022}\natexlab{}.
\newblock \showarticletitle{Federated reinforcement learning with environment heterogeneity}. In \bibinfo{booktitle}{\emph{International Conference on Artificial Intelligence and Statistics}}. \bibinfo{pages}{18--37}.
\newblock


\bibitem[\protect\citeauthoryear{Jordan, Gr\"{o}tschla, Fan, and Wattenhofer}{Jordan et~al\mbox{.}}{2024}]%
        {ETH.DECEN.BYZPG}
\bibfield{author}{\bibinfo{person}{Philip Jordan}, \bibinfo{person}{Florian Gr\"{o}tschla}, \bibinfo{person}{Flint~Xiaofeng Fan}, {and} \bibinfo{person}{Roger Wattenhofer}.} \bibinfo{year}{2024}\natexlab{}.
\newblock \showarticletitle{Decentralized Federated Policy Gradient with Byzantine Fault-Tolerance and Provably Fast Convergence}. In \bibinfo{booktitle}{\emph{International Conference on Autonomous Agents and Multiagent Systems}}. \bibinfo{pages}{964–972}.
\newblock


\bibitem[\protect\citeauthoryear{Ju, Juan, Gomez, Nakamura, and Li}{Ju et~al\mbox{.}}{2022}]%
        {ju2022transferring}
\bibfield{author}{\bibinfo{person}{Hao Ju}, \bibinfo{person}{Rongshun Juan}, \bibinfo{person}{Randy Gomez}, \bibinfo{person}{Keisuke Nakamura}, {and} \bibinfo{person}{Guangliang Li}.} \bibinfo{year}{2022}\natexlab{}.
\newblock \showarticletitle{Transferring policy of deep reinforcement learning from simulation to reality for robotics}.
\newblock \bibinfo{journal}{\emph{Nature Machine Intelligence}} \bibinfo{volume}{4}, \bibinfo{number}{12} (\bibinfo{year}{2022}), \bibinfo{pages}{1077--1087}.
\newblock


\bibitem[\protect\citeauthoryear{Khodadadian, Sharma, Joshi, and Maguluri}{Khodadadian et~al\mbox{.}}{2022}]%
        {khodadadian22a}
\bibfield{author}{\bibinfo{person}{Sajad Khodadadian}, \bibinfo{person}{Pranay Sharma}, \bibinfo{person}{Gauri Joshi}, {and} \bibinfo{person}{Siva~Theja Maguluri}.} \bibinfo{year}{2022}\natexlab{}.
\newblock \showarticletitle{Federated Reinforcement Learning: Linear Speedup Under {M}arkovian Sampling}. In \bibinfo{booktitle}{\emph{International Conference on Machine Learning}}, Vol.~\bibinfo{volume}{162}. \bibinfo{pages}{10997--11057}.
\newblock


\bibitem[\protect\citeauthoryear{Kingma}{Kingma}{2014}]%
        {kingma2014adam}
\bibfield{author}{\bibinfo{person}{Diederik~P Kingma}.} \bibinfo{year}{2014}\natexlab{}.
\newblock \showarticletitle{Adam: A method for stochastic optimization}.
\newblock \bibinfo{journal}{\emph{arXiv preprint arXiv:1412.6980}} (\bibinfo{year}{2014}).
\newblock


\bibitem[\protect\citeauthoryear{Li and Wang}{Li and Wang}{2019}]%
        {li2019fedmd}
\bibfield{author}{\bibinfo{person}{Daliang Li} {and} \bibinfo{person}{Junpu Wang}.} \bibinfo{year}{2019}\natexlab{}.
\newblock \showarticletitle{Fedmd: Heterogenous federated learning via model distillation}.
\newblock \bibinfo{journal}{\emph{arXiv preprint arXiv:1910.03581}} (\bibinfo{year}{2019}).
\newblock


\bibitem[\protect\citeauthoryear{Lin, Kong, Stich, and Jaggi}{Lin et~al\mbox{.}}{2020}]%
        {lin2020ensemble}
\bibfield{author}{\bibinfo{person}{Tao Lin}, \bibinfo{person}{Lingjing Kong}, \bibinfo{person}{Sebastian~U Stich}, {and} \bibinfo{person}{Martin Jaggi}.} \bibinfo{year}{2020}\natexlab{}.
\newblock \showarticletitle{Ensemble distillation for robust model fusion in federated learning}. In \bibinfo{booktitle}{\emph{NeurIPS 2020}}.
\newblock


\bibitem[\protect\citeauthoryear{Liu and Wang}{Liu and Wang}{2016}]%
        {NIPS2016_b3ba8f1b}
\bibfield{author}{\bibinfo{person}{Qiang Liu} {and} \bibinfo{person}{Dilin Wang}.} \bibinfo{year}{2016}\natexlab{}.
\newblock \showarticletitle{Stein Variational Gradient Descent: A General Purpose Bayesian Inference Algorithm}. In \bibinfo{booktitle}{\emph{NeurIPS}}, Vol.~\bibinfo{volume}{29}.
\newblock


\bibitem[\protect\citeauthoryear{Mai, Yao, Chen, Zhang, Cheung, and Han}{Mai et~al\mbox{.}}{2023}]%
        {mai2023server}
\bibfield{author}{\bibinfo{person}{Weiming Mai}, \bibinfo{person}{Jiangchao Yao}, \bibinfo{person}{Gong Chen}, \bibinfo{person}{Ya Zhang}, \bibinfo{person}{Yiu-Ming Cheung}, {and} \bibinfo{person}{Bo Han}.} \bibinfo{year}{2023}\natexlab{}.
\newblock \showarticletitle{Server-client collaborative distillation for federated reinforcement learning}.
\newblock \bibinfo{journal}{\emph{ACM Transactions on Knowledge Discovery from Data}} \bibinfo{volume}{18}, \bibinfo{number}{1} (\bibinfo{year}{2023}), \bibinfo{pages}{1--22}.
\newblock


\bibitem[\protect\citeauthoryear{Mak, Fan, Lanzend{\"o}rfer, Tan, Ooi, and Wattenhofer}{Mak et~al\mbox{.}}{2024}]%
        {mak2024caesar}
\bibfield{author}{\bibinfo{person}{Hei~Yi Mak}, \bibinfo{person}{Flint~Xiaofeng Fan}, \bibinfo{person}{Luca~A Lanzend{\"o}rfer}, \bibinfo{person}{Cheston Tan}, \bibinfo{person}{Wei~Tsang Ooi}, {and} \bibinfo{person}{Roger Wattenhofer}.} \bibinfo{year}{2024}\natexlab{}.
\newblock \showarticletitle{CAESAR: Enhancing Federated RL in Heterogeneous MDPs through Convergence-Aware Sampling with Screening}.
\newblock \bibinfo{journal}{\emph{arXiv preprint arXiv:2403.20156}} (\bibinfo{year}{2024}).
\newblock


\bibitem[\protect\citeauthoryear{McMahan, Moore, Ramage, Hampson, and y~Arcas}{McMahan et~al\mbox{.}}{2017}]%
        {mcmahan2017communication}
\bibfield{author}{\bibinfo{person}{Brendan McMahan}, \bibinfo{person}{Eider Moore}, \bibinfo{person}{Daniel Ramage}, \bibinfo{person}{Seth Hampson}, {and} \bibinfo{person}{Blaise~Aguera y Arcas}.} \bibinfo{year}{2017}\natexlab{}.
\newblock \showarticletitle{Communication-efficient learning of deep networks from decentralized data}. In \bibinfo{booktitle}{\emph{International Conference on Artificial Intelligence and Statistics}}. \bibinfo{pages}{1273--1282}.
\newblock


\bibitem[\protect\citeauthoryear{Mnih}{Mnih}{2016}]%
        {mnih2016asynchronous}
\bibfield{author}{\bibinfo{person}{Volodymyr Mnih}.} \bibinfo{year}{2016}\natexlab{}.
\newblock \showarticletitle{Asynchronous Methods for Deep Reinforcement Learning}.
\newblock \bibinfo{journal}{\emph{arXiv preprint arXiv:1602.01783}} (\bibinfo{year}{2016}).
\newblock


\bibitem[\protect\citeauthoryear{Pan, Wang, Zhang, Li, Yi, and Song}{Pan et~al\mbox{.}}{2019}]%
        {pan2019you}
\bibfield{author}{\bibinfo{person}{Xinlei Pan}, \bibinfo{person}{Weiyao Wang}, \bibinfo{person}{Xiaoshuai Zhang}, \bibinfo{person}{Bo Li}, \bibinfo{person}{Jinfeng Yi}, {and} \bibinfo{person}{Dawn Song}.} \bibinfo{year}{2019}\natexlab{}.
\newblock \showarticletitle{How You Act Tells a Lot: Privacy-Leaking Attack on Deep Reinforcement Learning.}. In \bibinfo{booktitle}{\emph{AAMAS 2019}}, Vol.~\bibinfo{volume}{19}. \bibinfo{pages}{368--376}.
\newblock


\bibitem[\protect\citeauthoryear{Papini, Binaghi, Canonaco, Pirotta, and Restelli}{Papini et~al\mbox{.}}{2018}]%
        {pmlr-v80-papini18a}
\bibfield{author}{\bibinfo{person}{Matteo Papini}, \bibinfo{person}{Damiano Binaghi}, \bibinfo{person}{Giuseppe Canonaco}, \bibinfo{person}{Matteo Pirotta}, {and} \bibinfo{person}{Marcello Restelli}.} \bibinfo{year}{2018}\natexlab{}.
\newblock \showarticletitle{Stochastic Variance-Reduced Policy Gradient}. In \bibinfo{booktitle}{\emph{International Conference on Machine Learning}}, Vol.~\bibinfo{volume}{80}. \bibinfo{pages}{4026--4035}.
\newblock


\bibitem[\protect\citeauthoryear{Qiao, Zhang, Yue, Yuan, Cai, Zhang, Ren, and Yu}{Qiao et~al\mbox{.}}{2024}]%
        {qiao2024br}
\bibfield{author}{\bibinfo{person}{Jing Qiao}, \bibinfo{person}{Zuyuan Zhang}, \bibinfo{person}{Sheng Yue}, \bibinfo{person}{Yuan Yuan}, \bibinfo{person}{Zhipeng Cai}, \bibinfo{person}{Xiao Zhang}, \bibinfo{person}{Ju Ren}, {and} \bibinfo{person}{Dongxiao Yu}.} \bibinfo{year}{2024}\natexlab{}.
\newblock \showarticletitle{BR-DeFedRL: Byzantine-Robust Decentralized Federated Reinforcement Learning with Fast Convergence and Communication Efficiency}. In \bibinfo{booktitle}{\emph{IEEE Conference on Computer Communications}}. IEEE, \bibinfo{pages}{141--150}.
\newblock


\bibitem[\protect\citeauthoryear{Reddi, Hefny, Sra, Poczos, and Smola}{Reddi et~al\mbox{.}}{2016}]%
        {pmlr-v48-reddi16}
\bibfield{author}{\bibinfo{person}{Sashank~J. Reddi}, \bibinfo{person}{Ahmed Hefny}, \bibinfo{person}{Suvrit Sra}, \bibinfo{person}{Barnabas Poczos}, {and} \bibinfo{person}{Alex Smola}.} \bibinfo{year}{2016}\natexlab{}.
\newblock \showarticletitle{Stochastic Variance Reduction for Nonconvex Optimization}. In \bibinfo{booktitle}{\emph{International Conference on Machine Learning}}, Vol.~\bibinfo{volume}{48}. \bibinfo{pages}{314--323}.
\newblock


\bibitem[\protect\citeauthoryear{Rusu, Colmenarejo, Gulcehre, Desjardins, Kirkpatrick, Pascanu, Mnih, Kavukcuoglu, and Hadsell}{Rusu et~al\mbox{.}}{2015}]%
        {rusu2015policy}
\bibfield{author}{\bibinfo{person}{Andrei~A Rusu}, \bibinfo{person}{Sergio~Gomez Colmenarejo}, \bibinfo{person}{Caglar Gulcehre}, \bibinfo{person}{Guillaume Desjardins}, \bibinfo{person}{James Kirkpatrick}, \bibinfo{person}{Razvan Pascanu}, \bibinfo{person}{Volodymyr Mnih}, \bibinfo{person}{Koray Kavukcuoglu}, {and} \bibinfo{person}{Raia Hadsell}.} \bibinfo{year}{2015}\natexlab{}.
\newblock \showarticletitle{Policy distillation}.
\newblock \bibinfo{journal}{\emph{arXiv preprint arXiv:1511.06295}} (\bibinfo{year}{2015}).
\newblock


\bibitem[\protect\citeauthoryear{Ryu and Takamaeda-Yamazaki}{Ryu and Takamaeda-Yamazaki}{2022}]%
        {ryu2022model}
\bibfield{author}{\bibinfo{person}{Sefutsu Ryu} {and} \bibinfo{person}{Shinya Takamaeda-Yamazaki}.} \bibinfo{year}{2022}\natexlab{}.
\newblock \showarticletitle{Model-based federated reinforcement distillation}. In \bibinfo{booktitle}{\emph{IEEE Global Communications Conference}}. \bibinfo{pages}{1109--1114}.
\newblock


\bibitem[\protect\citeauthoryear{Schulman, Levine, Abbeel, Jordan, and Moritz}{Schulman et~al\mbox{.}}{2015}]%
        {pmlr-v37-schulman15}
\bibfield{author}{\bibinfo{person}{John Schulman}, \bibinfo{person}{Sergey Levine}, \bibinfo{person}{Pieter Abbeel}, \bibinfo{person}{Michael Jordan}, {and} \bibinfo{person}{Philipp Moritz}.} \bibinfo{year}{2015}\natexlab{}.
\newblock \showarticletitle{Trust Region Policy Optimization}. In \bibinfo{booktitle}{\emph{International Conference on Machine Learning}}, Vol.~\bibinfo{volume}{37}. \bibinfo{pages}{1889--1897}.
\newblock


\bibitem[\protect\citeauthoryear{Schulman, Wolski, Dhariwal, Radford, and Klimov}{Schulman et~al\mbox{.}}{2017}]%
        {schulman2017proximal}
\bibfield{author}{\bibinfo{person}{John Schulman}, \bibinfo{person}{Filip Wolski}, \bibinfo{person}{Prafulla Dhariwal}, \bibinfo{person}{Alec Radford}, {and} \bibinfo{person}{Oleg Klimov}.} \bibinfo{year}{2017}\natexlab{}.
\newblock \showarticletitle{Proximal policy optimization algorithms}.
\newblock \bibinfo{journal}{\emph{arXiv preprint arXiv:1707.06347}} (\bibinfo{year}{2017}).
\newblock


\bibitem[\protect\citeauthoryear{Silver, Lever, Heess, Degris, Wierstra, and Riedmiller}{Silver et~al\mbox{.}}{2014}]%
        {silver2014deterministic}
\bibfield{author}{\bibinfo{person}{David Silver}, \bibinfo{person}{Guy Lever}, \bibinfo{person}{Nicolas Heess}, \bibinfo{person}{Thomas Degris}, \bibinfo{person}{Daan Wierstra}, {and} \bibinfo{person}{Martin Riedmiller}.} \bibinfo{year}{2014}\natexlab{}.
\newblock \showarticletitle{Deterministic policy gradient algorithms}. In \bibinfo{booktitle}{\emph{International Conference on Machine Learning}}. \bibinfo{pages}{387--395}.
\newblock


\bibitem[\protect\citeauthoryear{Sun and Lyu}{Sun and Lyu}{2021}]%
        {ijcaisun2021}
\bibfield{author}{\bibinfo{person}{Lichao Sun} {and} \bibinfo{person}{Lingjuan Lyu}.} \bibinfo{year}{2021}\natexlab{}.
\newblock \showarticletitle{Federated Model Distillation with Noise-Free Differential Privacy}. In \bibinfo{booktitle}{\emph{IJCAI}}. \bibinfo{pages}{1563--1570}.
\newblock


\bibitem[\protect\citeauthoryear{Sutton}{Sutton}{2018}]%
        {sutton2018reinforcement}
\bibfield{author}{\bibinfo{person}{Richard~S Sutton}.} \bibinfo{year}{2018}\natexlab{}.
\newblock \showarticletitle{Reinforcement learning: An introduction}.
\newblock \bibinfo{journal}{\emph{A Bradford Book}} (\bibinfo{year}{2018}).
\newblock


\bibitem[\protect\citeauthoryear{Touati, Zhang, Pineau, and Vincent}{Touati et~al\mbox{.}}{2020}]%
        {pmlr-v124-touati20a}
\bibfield{author}{\bibinfo{person}{Ahmed Touati}, \bibinfo{person}{Amy Zhang}, \bibinfo{person}{Joelle Pineau}, {and} \bibinfo{person}{Pascal Vincent}.} \bibinfo{year}{2020}\natexlab{}.
\newblock \showarticletitle{Stable Policy Optimization via Off-Policy Divergence Regularization}. In \bibinfo{booktitle}{\emph{Conference on Uncertainty in Artificial Intelligence}}, Vol.~\bibinfo{volume}{124}. \bibinfo{pages}{1328--1337}.
\newblock


\bibitem[\protect\citeauthoryear{Wang, Li, Xiong, and Zhang}{Wang et~al\mbox{.}}{2019}]%
        {NEURIPS2019_Divergence-Augmented}
\bibfield{author}{\bibinfo{person}{Qing Wang}, \bibinfo{person}{Yingru Li}, \bibinfo{person}{Jiechao Xiong}, {and} \bibinfo{person}{Tong Zhang}.} \bibinfo{year}{2019}\natexlab{}.
\newblock \showarticletitle{Divergence-Augmented Policy Optimization}. In \bibinfo{booktitle}{\emph{Advances in Neural Information Processing Systems}}, Vol.~\bibinfo{volume}{32}.
\newblock


\bibitem[\protect\citeauthoryear{Williams}{Williams}{1992}]%
        {williams1992simple}
\bibfield{author}{\bibinfo{person}{Ronald~J Williams}.} \bibinfo{year}{1992}\natexlab{}.
\newblock \showarticletitle{Simple statistical gradient-following algorithms for connectionist reinforcement learning}.
\newblock \bibinfo{journal}{\emph{Machine Learning}}  \bibinfo{volume}{8} (\bibinfo{year}{1992}), \bibinfo{pages}{229--256}.
\newblock


\bibitem[\protect\citeauthoryear{Woo, Joshi, and Chi}{Woo et~al\mbox{.}}{2023}]%
        {woo2023blessing}
\bibfield{author}{\bibinfo{person}{Jiin Woo}, \bibinfo{person}{Gauri Joshi}, {and} \bibinfo{person}{Yuejie Chi}.} \bibinfo{year}{2023}\natexlab{}.
\newblock \showarticletitle{The blessing of heterogeneity in federated q-learning: Linear speedup and beyond}. In \bibinfo{booktitle}{\emph{International Conference on Machine Learning}}. \bibinfo{pages}{37157--37216}.
\newblock


\bibitem[\protect\citeauthoryear{Xin, Karatzoglou, Arapakis, and Jose}{Xin et~al\mbox{.}}{2022}]%
        {xin2022supervised}
\bibfield{author}{\bibinfo{person}{Xin Xin}, \bibinfo{person}{Alexandros Karatzoglou}, \bibinfo{person}{Ioannis Arapakis}, {and} \bibinfo{person}{Joemon~M Jose}.} \bibinfo{year}{2022}\natexlab{}.
\newblock \showarticletitle{Supervised advantage actor-critic for recommender systems}. In \bibinfo{booktitle}{\emph{International Conference on Web Search and Data Mining}}. \bibinfo{pages}{1186--1196}.
\newblock


\bibitem[\protect\citeauthoryear{Xu, Gao, and Gu}{Xu et~al\mbox{.}}{2020}]%
        {xu20a_uai}
\bibfield{author}{\bibinfo{person}{Pan Xu}, \bibinfo{person}{Felicia Gao}, {and} \bibinfo{person}{Quanquan Gu}.} \bibinfo{year}{2020}\natexlab{}.
\newblock \showarticletitle{An Improved Convergence Analysis of Stochastic Variance-Reduced Policy Gradient}. In \bibinfo{booktitle}{\emph{Conference on Uncertainty in Artificial Intelligence}}, Vol.~\bibinfo{volume}{115}. \bibinfo{pages}{541--551}.
\newblock


\bibitem[\protect\citeauthoryear{Ye, Fang, Du, Yuen, and Tao}{Ye et~al\mbox{.}}{2023}]%
        {ye2023heterogeneous}
\bibfield{author}{\bibinfo{person}{Mang Ye}, \bibinfo{person}{Xiuwen Fang}, \bibinfo{person}{Bo Du}, \bibinfo{person}{Pong~C Yuen}, {and} \bibinfo{person}{Dacheng Tao}.} \bibinfo{year}{2023}\natexlab{}.
\newblock \showarticletitle{Heterogeneous federated learning: State-of-the-art and research challenges}.
\newblock \bibinfo{journal}{\emph{Comput. Surveys}} \bibinfo{volume}{56}, \bibinfo{number}{3} (\bibinfo{year}{2023}), \bibinfo{pages}{1--44}.
\newblock


\bibitem[\protect\citeauthoryear{Yu}{Yu}{2018}]%
        {yu2018towards}
\bibfield{author}{\bibinfo{person}{Yang Yu}.} \bibinfo{year}{2018}\natexlab{}.
\newblock \showarticletitle{Towards sample efficient reinforcement learning}. In \bibinfo{booktitle}{\emph{International Joint Conference on Artificial Intelligence}}. \bibinfo{pages}{5739--5743}.
\newblock


\bibitem[\protect\citeauthoryear{Yuan, Gower, and Lazaric}{Yuan et~al\mbox{.}}{2022}]%
        {pmlr-v151-yuan22a}
\bibfield{author}{\bibinfo{person}{Rui Yuan}, \bibinfo{person}{Robert~M. Gower}, {and} \bibinfo{person}{Alessandro Lazaric}.} \bibinfo{year}{2022}\natexlab{}.
\newblock \showarticletitle{A general sample complexity analysis of vanilla policy gradient}. In \bibinfo{booktitle}{\emph{International Conference on Artificial Intelligence and Statistics}}, Vol.~\bibinfo{volume}{151}. \bibinfo{pages}{3332--3380}.
\newblock


\bibitem[\protect\citeauthoryear{Zhang, Wang, Mitra, and Anderson}{Zhang et~al\mbox{.}}{2024}]%
        {zhang2024fedsarsa}
\bibfield{author}{\bibinfo{person}{Chenyu Zhang}, \bibinfo{person}{Han Wang}, \bibinfo{person}{Aritra Mitra}, {and} \bibinfo{person}{James Anderson}.} \bibinfo{year}{2024}\natexlab{}.
\newblock \showarticletitle{Finite-Time Analysis of On-Policy Heterogeneous Federated Reinforcement Learning}. In \bibinfo{booktitle}{\emph{International Conference on Learning Representations}}.
\newblock


\bibitem[\protect\citeauthoryear{Zheng, Gao, Xue, and Yang}{Zheng et~al\mbox{.}}{2024}]%
        {zheng2024federated}
\bibfield{author}{\bibinfo{person}{Zhong Zheng}, \bibinfo{person}{Fengyu Gao}, \bibinfo{person}{Lingzhou Xue}, {and} \bibinfo{person}{Jing Yang}.} \bibinfo{year}{2024}\natexlab{}.
\newblock \showarticletitle{Federated Q-Learning: Linear Regret Speedup with Low Communication Cost}. In \bibinfo{booktitle}{\emph{International Conference on Learning Representations}}.
\newblock


\bibitem[\protect\citeauthoryear{Zhu, Hong, and Zhou}{Zhu et~al\mbox{.}}{2021}]%
        {zhu2021data}
\bibfield{author}{\bibinfo{person}{Zhuangdi Zhu}, \bibinfo{person}{Junyuan Hong}, {and} \bibinfo{person}{Jiayu Zhou}.} \bibinfo{year}{2021}\natexlab{}.
\newblock \showarticletitle{Data-free knowledge distillation for heterogeneous federated learning}. In \bibinfo{booktitle}{\emph{International Conference on Machine Learning}}. \bibinfo{pages}{12878--12889}.
\newblock


\end{thebibliography}

%%%%%%%%%%%%%%%%%%%%%%%%%%%%%%%%%%%%%%%%%%%%%%%%%%%%%%%%%%%%%%%%%%%%%%%%

\appendix
\onecolumn

\section*{\huge APPENDIX}

\section{Proof}

\subsection{Proof of Lemma \ref{D_KL smooth}}\label{proof of lemma 1}
\begin{proof}
For different tasks, we use different distributions. Therefore, we need to separately prove the $L$-smoothness of the KL divergence term under the Softmax and Gaussian distributions:

\textit{1. $L$-smoothness of KL Divergence term under Softmax Distribution.}

Let $\pi_{\theta_k}(a|\mathcal{S}_p)$ be a Softmax distribution parameterized by $\theta_k$:
\begin{equation}
\pi_{\theta_k}(a|\mathcal{S}_p)=\frac{e^{\theta_k^Tf(a,\mathcal{S}_p)}}{\sum_je^{\theta_k^Tf(a_j,\mathcal{S}_p)}}.
\end{equation}

The KL divergence between $\pi_{\theta_k}$ and $\pi_{global}$ is:
\begin{equation}
D_{\mathrm{KL}}(\pi_{\theta_k}\|\pi_{global})=\sum_a\pi_{\theta_k}(a|\mathcal{S}_p)\log\frac{\pi_{\theta_k}(a|\mathcal{S}_p)}{\pi_{global}(a|\mathcal{S}_p)} .
\end{equation}

The gradient of KL divergence with respect to $\theta_k$ can be derived as follows:

First, compute the gradient of $\pi_{\theta_k}(a|\mathcal{S}_p)$ with respect to $\theta_k$:
\begin{align}
\nabla_{\theta_k}\pi_{\theta_k}(a|\mathcal{S}_p)=\pi_{\theta_k}(a|\mathcal{S}_p)\left(f(a,\mathcal{S}_p)-\sum_j\pi_{\theta_k}(a_j|\mathcal{S}_p)f(a_j,\mathcal{S}_p)\right).  
\end{align}

Next, compute the gradient of $D_{\text{KL}}(\pi_{\theta_k} \| \pi_{global})$:
\begin{align}
\nabla_{\theta_k}D_{\text{KL}}(\pi_{\theta_k} \| \pi_{global})=\sum_{a}(\nabla_{\theta_{k}}\pi_{\theta_{k}}(a|\mathcal{S}_{p})\log\frac{\pi_{\theta_{k}}(a|\mathcal{S}_{p})}{\pi_{global}(a|\mathcal{S}_{p})}+\pi_{\theta_{k}}(a|\mathcal{S}_{p})\nabla_{\theta_{k}}\log\frac{\pi_{\theta_{k}}(a|\mathcal{S}_{p})}{\pi_{global}(a|\mathcal{S}_{p})}) . 
\end{align}

Since $\nabla_{\theta_k} \log \frac{\pi_{\theta_k}(a|\mathcal{S}_p)}{\pi_{global}(a|\mathcal{S}_p)} = \frac{\nabla_{\theta_k} \pi_{\theta_k}(a|\mathcal{S}_p)}{\pi_{\theta_k}(a|\mathcal{S}_p)}$, it simplifies to:
\begin{align}
\nabla_{\theta_k}D_{\mathrm{KL}}(\pi_{\theta_k}\|\pi_{global})=\sum_a\left(f(a,\mathcal{S}_p)-\sum_j\pi_{\theta_k}(a_j|\mathcal{S}_p)f(a_j,\mathcal{S}_p)\right)\pi_{\theta_k}(a|\mathcal{S}_p). 
\end{align}

$\nabla_{\theta_k} D_{\text{KL}}(\pi_{\theta_k} \| \pi_{global})$ is Lipschitz continuous which means $\|\nabla_{\theta_k} D_{\text{KL}}(\pi_{\theta_k} \| \pi_{global}) - \nabla_{\theta_{k'}} D_{\text{KL}}(\pi_{\theta_{k'}} \| \pi_{global})\|$ is to be bounded. 

Let $\nabla_{\theta_k} D_{\text{KL}}(\pi_{\theta_k} \| \pi_{global})$ and $\nabla_{\theta_{k'}} D_{\text{KL}}(\pi_{\theta_{k'}} \| \pi_{global})$ be denoted by $\nabla_{\theta_k}$ and $\nabla_{\theta_{k'}}$, respectively. 
The difference can be written as:
\begin{align}
\|\nabla_{\theta_k}-\nabla_{\theta_{k^{\prime}}}\|=\|\sum_a\left(\pi_{\theta_k}(a|\mathcal{S}_p)-\pi_{\theta_{k^{\prime}}}(a|\mathcal{S}_p)\right)\left(f(a,\mathcal{S}_p)-\sum_j\pi_{\theta_k}(a_j|\mathcal{S}_p)f(a_j,\mathcal{S}_p)\right)\|. 
\end{align}

By the smoothness of the Softmax function and its gradient, there exists a constant $L_{KL}$ such that:
\begin{equation}
\|\nabla_{\theta_k} - \nabla_{\theta_{k'}}\| \leq L_{KL} \|\theta_k - \theta_{k'}\|,  
\end{equation}
where $L_{KL}$ is a constant dependent on the parameters of the Softmax function and the feature space.

\textit{2. $L$-smoothness of KL Divergence term under Gaussian Distribution.}

Let $\pi_{\theta_k}(a|\mathcal{S}_p)$ be the probability density function of a Gaussian distribution with mean $\mu_k(\theta)$ and variance $\sigma^2_k(\theta)$ for each component $k$:
\begin{equation}
\pi_{\theta_k}(a|\mathcal{S}_p) = \frac{1}{\sqrt{2 \pi \sigma^2(\theta_k)}} e^{\left(-\frac{(a - \mu(\theta_k))^2}{2 \sigma^2(\theta_k)}\right)}.     
\end{equation}

The global distribution $\pi_{global}(a|\mathcal{S}_p)$ is the average of these Gaussian components:
\begin{equation}
\pi_{global}(a|\mathcal{S}_p) = \frac{1}{K} \sum_{k=1}^K \pi_{\theta_k}(a|\mathcal{S}_p).   
\end{equation}

The KL divergence between $\pi_{\theta_k}(a|\mathcal{S}_p)$ and $\pi_{global}(a|\mathcal{S}_p)$ is given by:
\begin{equation}
D_{\text{KL}}(\pi_{\theta_k} \| \pi_{global}) = \int_{-\infty}^{\infty} \pi_{\theta_k}(a|\mathcal{S}_p) \log \frac{\pi_{\theta_k}(a|\mathcal{S}_p)}{\pi_{global}(a|\mathcal{S}_p)} \, da. 
\end{equation}

The KL divergence between two univariate Gaussian distributions $\mathcal{N}(\mu_1, \sigma_1^2)$ and $\mathcal{N}(\mu_2, \sigma_2^2)$ has a well-known closed form:
\begin{equation}\label{eq: KL divergence between two univariate Gaussian distributions}
D_{\mathrm{KL}}(\mathcal{N}(\mu_{1},\sigma_{1}^{2})\|\mathcal{N}(\mu_{2},\sigma_{2}^{2}))=\frac{1}{2}\left(\frac{\sigma_{1}^{2}}{\sigma_{2}^{2}}+\frac{(\mu_{2}-\mu_{1})^{2}}{\sigma_{2}^{2}}-1+\log\frac{\sigma_{2}^{2}}{\sigma_{1}^{2}}\right).
\end{equation}

Integrating over $a$ and using Eq.~(\ref{eq: KL divergence between two univariate Gaussian distributions}), we obtain:
\begin{align}
D_{\text{KL}}(\pi_{\theta_k} \| \pi_{global}) = \frac12[\frac1K\sum_{j=1}^K\frac{\sigma^2(\theta_j)}{\sigma^2(\theta_k)}+\frac{\left(\mu(\theta_k)-\frac1K\sum_{j=1}^K\mu(\theta_j)\right)^2}{\frac1K\sum_{j=1}^K\sigma^2(\theta_j)}-1-\log\frac{\sigma^2(\theta_k)}{\frac1K\sum_{j=1}^K\sigma^2(\theta_j)}].  
\end{align}

The gradient of $D_{\text{KL}}(\pi_{\theta_k} \| \pi_{global})$ with respect to $\theta_k$ is:
\begin{align}
\nabla_{\theta_k}D_{\text{KL}}(\pi_{\theta_k} \| \pi_{global})&=\frac{1}{2}[-\frac{1}{K}\sum_{j=1}^{K}\frac{\sigma^{2}(\theta_{j})}{\sigma^{4}(\theta_{k})} \nabla_{\theta_{k}}\sigma^{2}(\theta_{k})\notag\\&+\frac{2(\mu(\theta_k)-\frac{1}{K}\sum_{j=1}^{K}\mu(\theta_{j}))\nabla_{\theta_{k}}\mu(\theta_{k})\cdot\frac{1}{K}\sum_{j=1}^{K}\sigma^{2}(\theta_{j})-(\mu(\theta_{k})-\frac{1}{K}\sum_{j=1}^{K}\mu(\theta_{j}))^{2}\frac{1}{K}\sum_{j=1}^{K}\nabla_{\theta_{k}}\sigma^{2}(\theta_{j})}{\left(\frac{1}{K}\sum_{j=1}^{K}\sigma^{2}(\theta_{j})\right)^{2}}\notag\\&-\frac1{\sigma^2(\theta_k)}\nabla_{\theta_k}\sigma^2(\theta_k).
\end{align}

Since $\mu(\theta_k)$ and $\sigma^2(\theta_k)$ are $L$-smooth functions, their gradients are Lipschitz continuous. Therefore, the gradient of the KL divergence, which is composed of these smooth functions, is also Lipschitz continuous.
Thus, the KL divergence term $D_{\text{KL}}(\pi_{\theta_k} \| \pi_{global})$ is $L$-smooth with respect to $\theta_k$, where $L$ is a Lipschitz constant derived from the smoothness of $\sigma^2(\theta_k)$ and $\mu(\theta_k)$.
 
\end{proof}

\subsection{Proof of Theorem \ref{theorem: convergence of PG with KD}}\label{proof: theorem 1}
\begin{proof}

The gradient of the objective function $J'({\theta_k})$ with respect to $\theta_k$ can be formulated as: 
\begin{equation}\label{new gradient}
\nabla_{\theta_k} J'({\theta_k}) = \nabla_{\theta_k} J({\theta_k})- \lambda~\nabla_{\theta_k} D_{\text{KL}}(\pi_{\theta_k}(a|\mathcal{S}_p) \| \pi_{global}(a|\mathcal{S}_p)).    
\end{equation}

First, we consider the smoothness of the original objective function $J({\theta_k})$. For the sake of brevity, we replace $\theta_k$ with $\theta$:
\begin{equation}
J(\theta)-J(\theta^{\prime})\geq\nabla J(\theta^{\prime})^\top(\theta-\theta^{\prime})-\frac {L_J}{2}\|\theta-\theta^{\prime}\|^2.  
\end{equation}

Similarly, the KL divergence term satisfies:
\begin{equation}
D_{{\mathrm{KL}}}(\pi_{\theta}\parallel\pi_{global})-D_{{\mathrm{KL}}}(\pi_{{\theta^{\prime}}}\parallel\pi_{global})\geq\nabla D_{{\mathrm{KL}}}(\theta^{\prime})^{\top}(\theta-\theta^{\prime})-\frac {L_{KL}}{2}\|\theta-\theta^{\prime}\|^2.
\end{equation}

If the new objective function $J'(\theta)$ satisfies $L$-smoothness, it needs to satisfy:
\begin{equation}\label{J' smooth}
\begin{aligned}
J^{\prime}(\theta)-J^{\prime}(\theta^{\prime})\geq\nabla J^{\prime}(\theta^{\prime})^\top(\theta-\theta^{\prime})-\frac{L}2\|\theta-\theta^{\prime}\|^2,
\end{aligned}
\end{equation}
where $L = L_J - \lambda L_{KL} > 0$.

Therefore, we need to conduct a detailed analysis of the upper bound of $L$:
\begin{align}\label{eq: upper bound of L_J}
L_J = \frac{\mathcal R_{\max}}{(1-\gamma)^2}\left(G^2+M\right),    
\end{align}
where $\mathcal R_{\max}$ denotes the upper bound of the single-step reward.
Eq.~(\ref{eq: upper bound of L_J}) represents the current state-of-the-art theoretical result (Lemma 4.4 in \cite{pmlr-v151-yuan22a}).

Taking the KL divergence under the softmax distribution as an example:
\begin{align}
\nabla_{\theta}D_{KL}(\pi_{\theta}||\pi_{global})& =\nabla_\theta\left[\sum_{s,a}\pi_\theta(a|s)\log\left(\frac{\pi_\theta(a|s)}{\pi_{global}(a|s)}\right)\right] \notag\\
&=\sum_{s,a}\left[\nabla_\theta\pi_\theta(a|s)(1+\log\left(\frac{\pi_\theta(a|s)}{\pi_{global}(a|s)}\right))+\pi_\theta(a|s)\nabla_\theta\log\pi_\theta(a|s)\right] \notag\\
&=\sum_{s,a}\left[\pi_\theta(a|s)\nabla_\theta\log\pi_\theta(a|s)(2+\log\biggl(\frac{\pi_\theta(a|s)}{\pi_{global}(a|s)}\biggr))\right].
\end{align}

According to the Cauchy-Schwarz inequality:
\begin{align}
|\nabla_{\theta}D_{KL}(\pi_{\theta}\|\pi_{global})|&\leq\sum_{s,a}\pi_{\theta}(a|s) |\nabla_{\theta}\log\pi_{\theta}(a|s)|\left|2+\log\biggl(\frac{\pi_{\theta}(a|s)}{\pi_{global}(a|s)}\biggr)\right| \notag\\
&\leq G\sum_{s,a}\pi_{\theta}(a|s)|\left|2+\log\biggl(\frac{\pi_{\theta}(a|s)}{\pi_{global}(a|s)}\biggr)\right|. 
\end{align}

Using Jensen's inequality:
\begin{equation}
\sum_{s,a}\pi_\theta(a|s)\log\biggl(\frac{\pi_\theta(a|s)}{\pi_{global}(a|s)}\biggr)\leq\log\biggl(\sum_{s,a}\pi_\theta(a|s)\left(\frac{\pi_\theta(a|s)}{\pi_{global}(a|s)}\right)\biggr).
\end{equation}

If we consider a worst-case scenario, assume that $\frac{\pi_\theta(a|s)}{\pi_{global}(a|s)}$ is bounded by some constant $C$ for any state-action pair $(s, a)$:
\begin{align}
\frac{\pi_\theta(a|s)}{\pi_{global}(a|s)} \leq C.  
\end{align}

Then we have:
\begin{equation}
\sum_{s,a} \pi_\theta(a|s) \left(\frac{\pi_\theta(a|s)}{\pi_{global}(a|s)}\right) \leq \sum_{s,a} \pi_\theta(a|s) \cdot C = C,     
\end{equation}
where $C$ can be related to the size of the state space $|S|$ or the action space $|A|$. 

To simplify the analysis, we can assume $C$ is a function of $|A|$:
\begin{equation}
\sum_{s,a} \pi_\theta(a|s) \left(\frac{\pi_\theta(a|s)}{\pi_{global}(a|s)}\right) \leq |A|.    
\end{equation}

Therefore, $L_{KL}$ can be expressed as:
\begin{equation}
L_{KL} = G (2 + \log(|A|)) .   
\end{equation}

Due to the influence of the reward and discount factor, $L_J$ is generally greater than $L_{KL}$. 
This implies that in long-term tasks (with $\gamma$ close to 1) and larger reward magnitudes (larger $\mathcal R_{\max}$), the condition for $L > 0$ is more easily satisfied.

The update of the objective function $J'(\theta)$ is:
\begin{equation}\label{update of of J'}
\theta^{i+1}=\theta^i+\alpha^i\nabla J'(\theta^i).    
\end{equation}

According to the definition of $L$-smoothness of the objective function $J'(\theta)$, we can estimate the change in the gradients between two points: 
\begin{align}\label{gradient change between 2 points}
J^{\prime}(\theta^{i+1})-J^{\prime}(\theta^{i})\geq\nabla J^{\prime}(\theta^{i})^\top(\theta^{i+1}-\theta^{i})-\frac{L}2\|\theta^{i+1}-\theta^{i}\|^2.
\end{align}

Combining Eq.~(\ref{update of of J'}) and Eq.~(\ref{gradient change between 2 points}), we have:
\begin{align}
J'(\theta^{i+1}) - J'(\theta^i)&\geq \alpha^i\|\nabla J'(\theta^i)\|^2-\frac{L}{2} {\alpha^i}^2\|\nabla J'(\theta^i)\|^2
\end{align}

Summing over $i$ from $1$ to $H-1$:
\begin{equation}
J^{\prime}(\theta^{H})-J^{\prime}(\theta^1)\geq\sum_{i=1}^{H-1}\left(\alpha^i-\frac{L{\alpha^i}^2}2\right)\|\nabla J^{\prime}(\theta^i)\|^2.    
\end{equation}

To ensure that $J'(\theta^i)$ is bounded, $\sum_{i=1}^{H-1}\left(\alpha^i-\frac{L{\alpha^i}^2}2\right)\|\nabla J^{\prime}(\theta^i)\|^2$ is need to be bounded.

By choosing $\alpha^i$ such that $\alpha^i = \frac{1}{i}$ or other suitable sequences satisfy Robbins-Monro conditions:
\begin{equation}\label{Robbins-Monro conditions}
\sum_{i=1}^{\infty} \alpha^i = \infty , \quad \sum_{i=1}^{\infty} {\alpha^i}^2 < \infty.    
\end{equation}

For sufficiently large $i$, $\alpha^i$ will be small enough such that $1 - \frac{L \alpha^i}{2}$ remains positive. Thus, $\alpha^i - \frac{L {\alpha^i}^2}{2}$ is positive.

The term $\|\nabla J'(\theta^i)\|^2$ must tend to zero as $H \to \infty$. Otherwise, the sum would diverge, contradicting the boundedness. 
Hence, as $\|\nabla J'(\theta^i)\|^2 \to 0$, it follows that $\nabla J'(\theta^i) \to 0$, which completes the proof. 
\end{proof}

%\subsection{Proof of Corollary \ref{fast convergence FedHPD}}\label{proof:corollary 1}

\subsection{Supplementary Proof of Corollary \ref{fast convergence FedHPD} (The value of $\lambda$)}\label{Discussion lamda value}
In this section, we discuss why FedHPD sets $\lambda = 1$. 

Let $\mathrm{Var}\left[\nabla_\theta J(\theta)\right]+\mathrm{Var}[\nabla_\theta D_{\mathrm{KL}}(\pi_\theta||\pi_{\mathrm{global}})]$ be $x$ and $\mathrm{Cov}(\nabla_\theta J(\theta),\nabla_\theta D_\mathrm{KL}(\pi_\theta||\pi_\mathrm{global}))$ be $y$. 
Eq.~(\ref{eq: variance of the new gradient}) can be reformulated as:
\begin{align}\label{eq: discussion of lamda}
\mathrm{Var}[\nabla_\theta J^{\prime}(\theta)]&=\mathrm{Var}\left[\nabla_\theta J(\theta)\right]+{\lambda^2}x-2\lambda y \notag\\
&=\mathrm{Var}\left[\nabla_\theta J(\theta)\right]+{\lambda^2}(x-y)+{\lambda^2}y-2\lambda y \notag\\
&=\mathrm{Var}\left[\nabla_\theta J(\theta)\right]+{\lambda^2}(x-y)+y[{(\lambda-1)}^2-1].
\end{align}

To accelerate the algorithm’s convergence, we need to reduce the variance of the gradient, so ${\lambda^2}(x-y)+y[{(\lambda-1)}^2-1]$ should be as negative as possible.
Consider the following three cases: 

\textit{Case 1} ($x > y$): 
Since $x > y$, ${\lambda^2}(x-y) > 0$. 
To make ${\lambda^2}(x-y)+y[{(\lambda-1)}^2-1] < 0$, we need to minimize the value of ${(\lambda-1)}^2-1$. 
Thus, when $\lambda = 1$, Eq.~(\ref{eq: discussion of lamda}) can be reformulated as:
\begin{equation}
\mathrm{Var}[\nabla_\theta J^{\prime}(\theta)]=\mathrm{Var}\left[\nabla_\theta J(\theta)\right]+x-2y.    
\end{equation}

If $y < x \leq 2y$, it can still reduce or maintain the original variance.

\textit{Case 2} ($x = y$): Since $x = y$, Eq.~(\ref{eq: discussion of lamda}) can be reformulated as:
\begin{equation}
\mathrm{Var}[\nabla_\theta J^{\prime}(\theta)]=\mathrm{Var}\left[\nabla_\theta J(\theta)\right]+y[{(\lambda-1)}^2-1].    
\end{equation}

To minimize the gradient variance as much as possible, we still keep $\lambda = 1$.

\textit{Case 3} ($x < y$): Consistent with Case 2, $\lambda = 1$ is the optimal choice for minimizing the gradient variance.

In conclusion, we choose $\lambda = 1$ in FedHPD.
\clearpage

\section{Experimental details}\label{Appendix:Configurations}
The hyperparameter settings for different tasks in the experiments are shown in Table~\ref{tab:Hyperparameters}. 
Depending on the task type and difficulty, the number of training rounds is adjusted accordingly. 
Additionally, we choose the size of the public state set to balance communication cost and training utility \cite{ijcaisun2021}.

\begin{table}[h]
    \caption{Hyperparameters Used in Experiments.}
    \label{tab:Hyperparameters}
    \centering 
	\begin{tabular}{cccccc}\toprule
		\textit{Hyperparameters} & \textit{Cartpole} & \textit{LunarLander} & \textit{InvertedPendulum} \\ \midrule
	    Policy function & \makecell{Categorical\\MLP} & \makecell{Categorical\\MLP} & \makecell{Gaussian\\MLP} \\
		Discount factor & 0.99 & 0.99 & 0.99\\
        Training rounds & 2000 & 3000 & 6000 \\
		State set size & 5000 & 10000 & 5000 \\
        \bottomrule
	\end{tabular}
\end{table}

In our experiments, we parameterize the policies of heterogeneous agents as neural networks. 
16$\times$16$\times$32 (tanh, tanh, tanh) represents a neural network with three hidden layers, with 16, 16, and 32 neurons in each layer respectively, using the tanh activation function. 
The complete configurations for the agents is shown in Tables \ref{tab:Experiment Cartpole Configurations}, \ref{tab:Experiment LunarLander Configurations}, \ref{tab:Experiment InvertedPendulum Configurations}.
We use the Adam \cite{kingma2014adam} optimizer throughout the training process.
Furthermore, the experiments are conducted on a 13th Gen Intel(R) Core(TM) i9-13900H processor with a clock speed of 2.60GHz and 32GB of memory.
On the specified hardware, for the CartPole, LunarLander, and InvertedPendulum tasks, each individual experiment can be completed in less than 4, 10, and 12 hours, respectively.

\begin{table}[h]
    \begin{minipage}{0.48\linewidth}
	\caption{Configurations of Agents in Cartpole}
	\label{tab:Experiment Cartpole Configurations}
	\begin{tabular}{rlllll}\toprule
		\textit{Agent id} & \textit{Network} & \textit{Learning Rate} \\ \midrule
		1 & 128 (relu) & 1e-3 \\
		2 & 32$\times$32 (relu, relu) & 2e-3 \\
        3 & 16$\times$16$\times$32 (tanh, tanh, tanh) & 4e-3 \\
		4 & 8$\times$8$\times$8 (relu, relu, relu) & 5e-4 \\
		5 & 32$\times$32$\times$32 (tanh, tanh, tanh) & 3e-3 \\
        6 & 8$\times$8 (relu, relu) & 7e-4 \\
		7 & 64$\times$64 (tanh, tanh) & 1e-3 \\
        8 & 16$\times$16 (relu, relu) & 5e-4 \\
		9 & 16$\times$32$\times$16 (tanh, tanh, tanh) & 5e-4 \\
		10 & 32 (relu) & 8e-4 \\ \bottomrule
	\end{tabular}
    \end{minipage}
    \begin{minipage}{0.48\linewidth}
	\caption{Configurations of Agents in LunarLander}
	\label{tab:Experiment LunarLander Configurations}
	\begin{tabular}{rlllll}\toprule
		\textit{Agent id} & \textit{Network} & \textit{Learning Rate} \\ \midrule
		1 & 128$\times$128$\times$256 (relu, relu, relu) & 5e-4 \\
		2 & 64$\times$64 (relu, relu) & 1e-3 \\
        3 & 128$\times$128 (tanh, tanh) & 4e-4 \\
		4 & 128$\times$256 (relu, relu) & 6e-4 \\
		5 & 256$\times$256 (tanh, tanh) & 3e-4 \\
        6 & 512 (relu) & 4e-4 \\
		7 & 64$\times$128$\times$64 (tanh, tanh, tanh) & 5e-4 \\
        8 & 32$\times$32 (relu, relu) & 2e-3 \\
		9 & 512$\times$512 (tanh, tanh) & 2e-4 \\
		10 & 1024 (relu) & 1e-4 \\ \bottomrule
	\end{tabular}
    \end{minipage}
\end{table}

\begin{table}[h]
	\caption{Configurations of Agents in InvertedPendulum}
	\label{tab:Experiment InvertedPendulum Configurations}
	\begin{tabular}{rlllll}\toprule
		\textit{Agent id} & \textit{Network} & \textit{Learning Rate} \\ \midrule
		1 & 16$\times$32 (tanh, tanh) & 1e-4 \\
		2 & 32$\times$32 (relu, relu) & 1e-4 \\
        3 & 64$\times$128 (tanh, relu) & 6e-5 \\
		4 & 128$\times$256 (relu, relu) & 1e-5 \\
		5 & 32$\times$64 (relu, tanh) & 1e-4 \\
        6 & 64$\times$64 (tanh, tanh) & 8e-5 \\
		7 & 128$\times$128 (relu, relu) & 4e-5 \\
        8 & 64$\times$32 (tanh, relu) & 7e-5 \\
		9 & 256$\times$128 (tanh, tanh) & 2e-5 \\
		10 & 32$\times$128 (relu, relu) & 5e-5 \\ \bottomrule
	\end{tabular}
\end{table}
\clearpage

\section{Additional Experiments}\label{Appendix:FuLL Experiments}
\subsection{Comparisons of individual performance improvement by FedHPD ($d = 5, 10, 20$)}\label{Appendix C.1}
\begin{figure*}[h]	
	\begin{minipage}{1\linewidth}
	\centerline{\includegraphics[width=\textwidth]{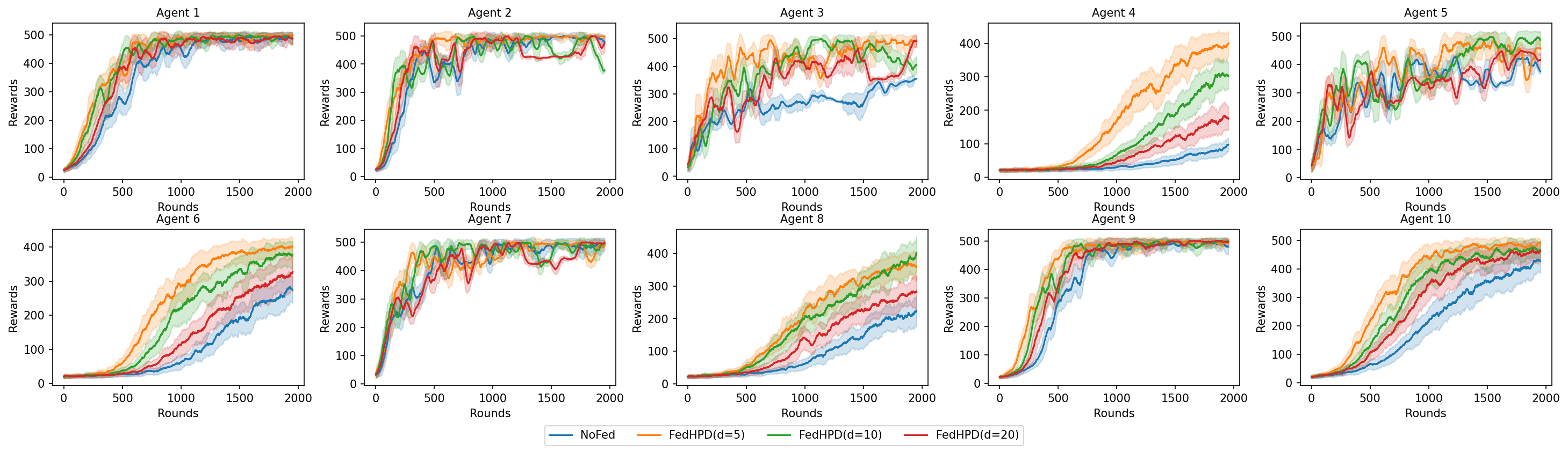}}	\centerline{\includegraphics[width=\textwidth]{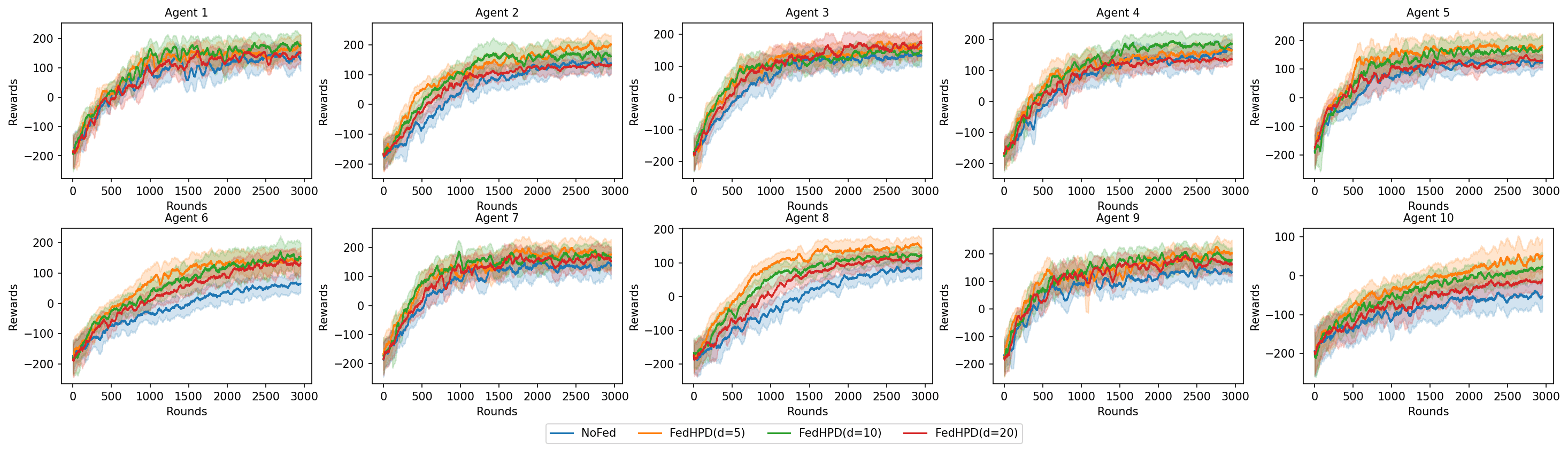}}
    \centerline{\includegraphics[width=\textwidth]{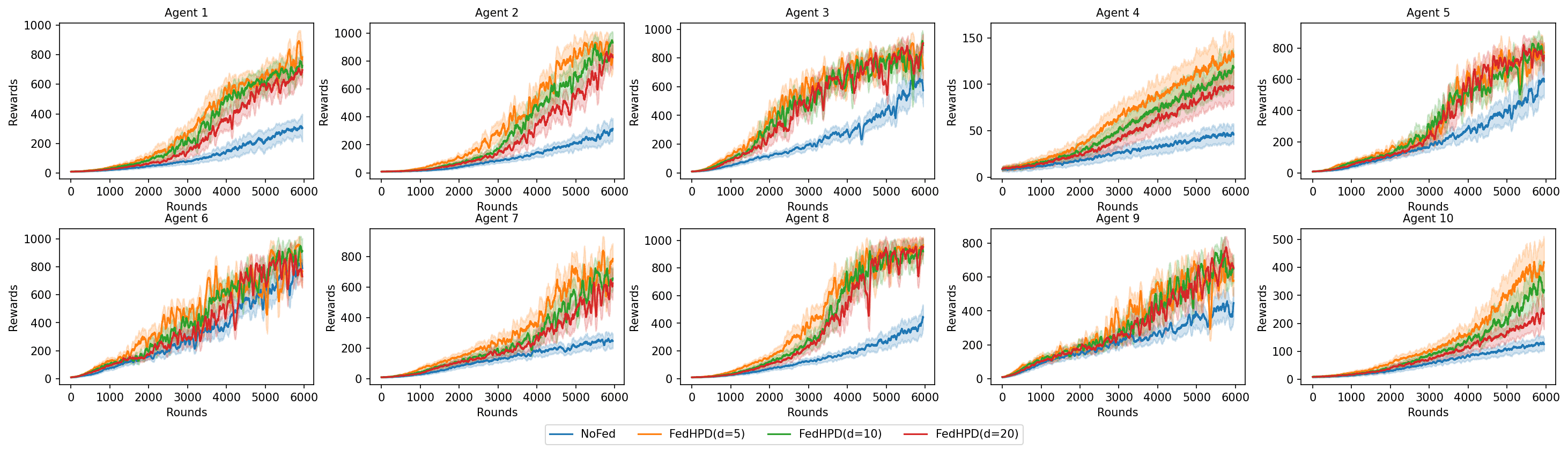}}
	\end{minipage}
	\caption{Complete comparisons of individual agents with different distillation intervals ($d$ = 5, 10 ,20).}
    \label{pic:complete individual}
\end{figure*} 

In Fig.~\ref{pic:complete individual}, we present the learning curves for all agents in the system, including performance comparisons between NoFed and FedHPD with different distillation interval ($d$ = 5, 10 ,20).
\clearpage
\subsection{Comparisons of individual performance between DPA-FedRL and FedHPD ($d$ = 5, 10, 20)}\label{Appendix: FedHPD baselines}
Fig. \ref{pic:Comparisons of individual agent baseline.} compares the individual improvement between DPA-FedRL and FedHPD.
DPA-FedRL can effectively maintain well-parameterized agents and is more stable compared to FedHPD.
However, for poorly parameterized or unstable agents, DPA-FedRL fails to improve their sample efficiency and may even cause them to converge to worse policies, which contradicts our problem objective.

\begin{figure}[h]	
	\begin{minipage}{1\linewidth}
	\centerline{\includegraphics[width=\textwidth]{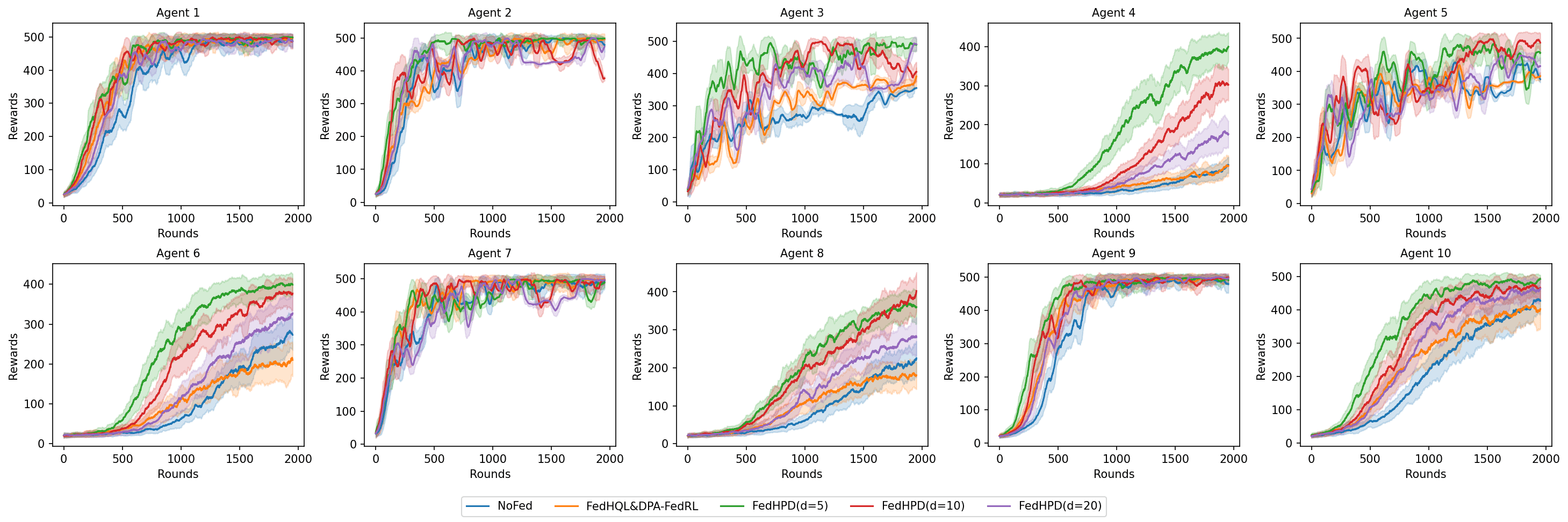}}	\centerline{\includegraphics[width=\textwidth]{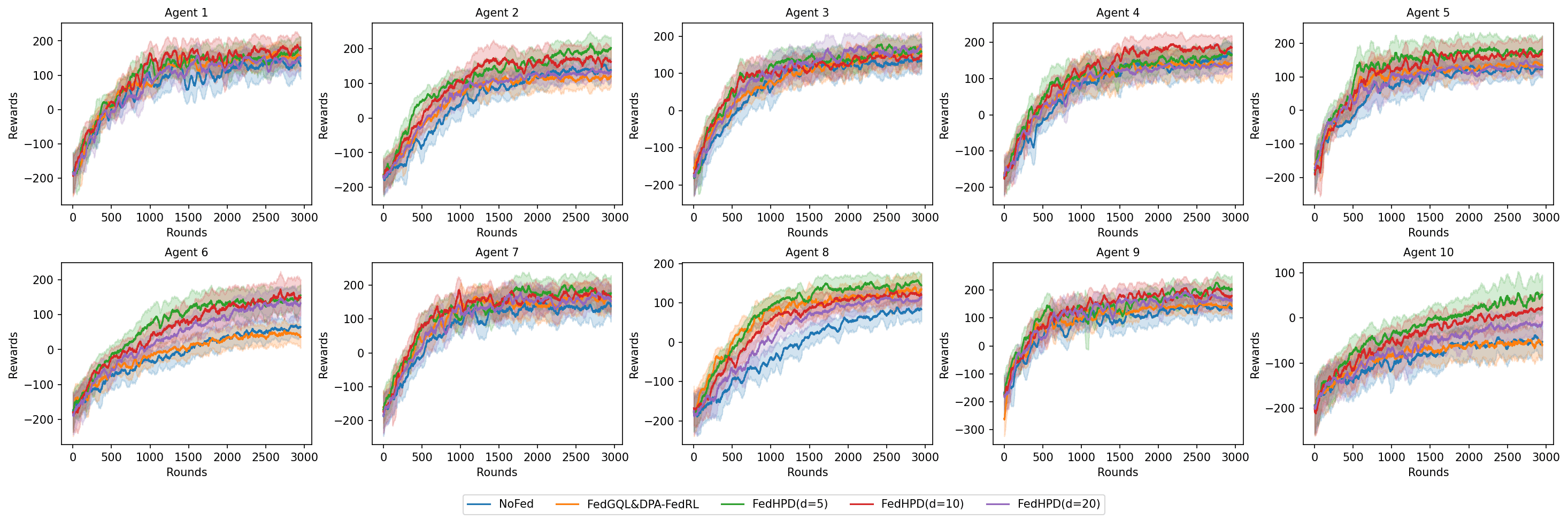}}
	\end{minipage}
	\caption{Comparisons of individual agent between DPA-FedRL and FedHPD ($d$ = 5, 10, 20).}
    \label{pic:Comparisons of individual agent baseline.}
\end{figure} 
\clearpage

\subsection{Comparisons of individual performance improvement by FedHPD ($d$ = 2, 5, 10, 20, 40, 80)}\label{Appendix: FedHPD under six d values}
\begin{figure}[h]	
	\begin{minipage}{1\linewidth}
	\centerline{\includegraphics[width=\textwidth]{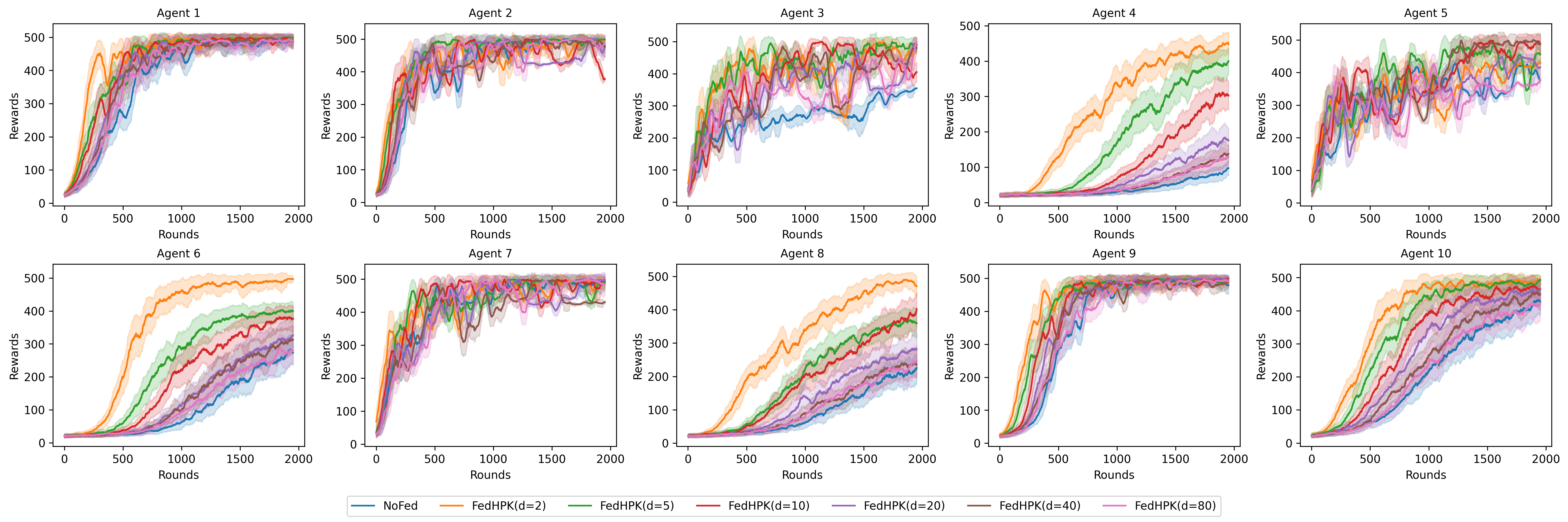}}	\centerline{\includegraphics[width=\textwidth]{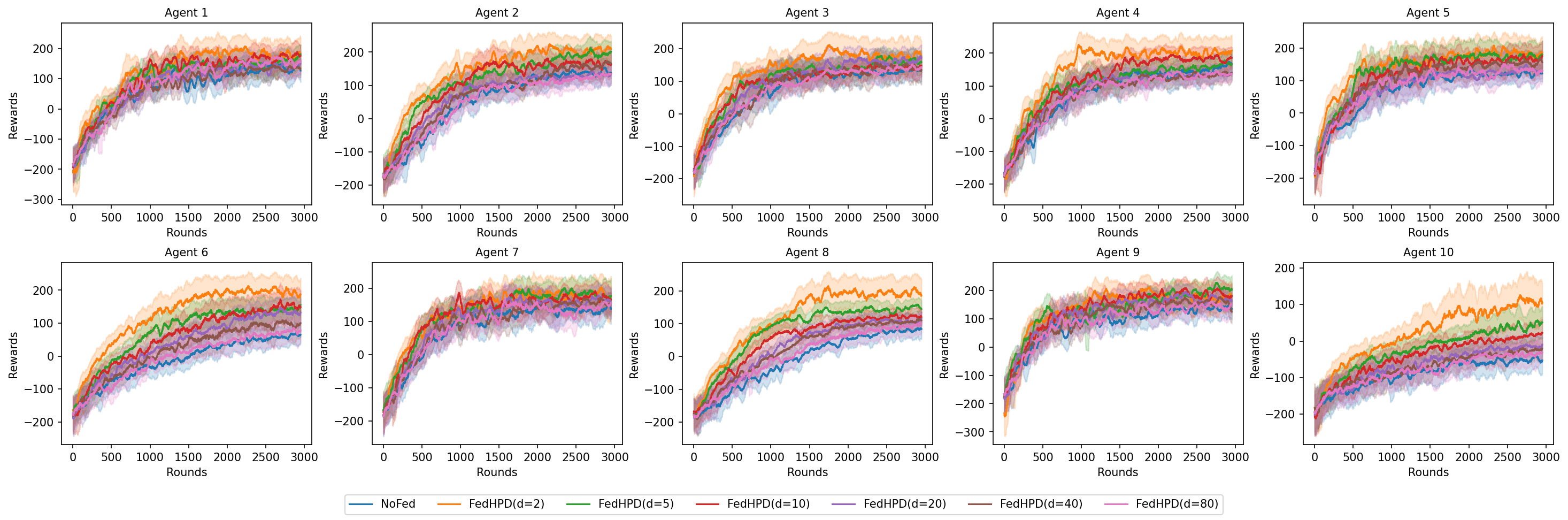}}
	\end{minipage}
	\caption{Comparisons of individual agent with different distillation intervals in Cartpole ($d$ = 2, 5, 10, 20, 40, 80).}
    \label{pic:Comparisons of individual agent with six d values.}
\end{figure} 

Fig. \ref{pic:Comparisons of individual agent with six d values.} illustrates the comparative performance of individual agents.
Firstly, we can determine that most agents trained using the FedHPD exhibit better performance compared to NoFed.
However, increasing the distillation interval $d$ may lead some agents towards suboptimal strategies, for example: In Cartpole, Agent 5 at $d$ = 80 and Agent 7 at $d$ = 40 show performance inferior to NoFed.
We hypothesize that the reason for this phenomenon may be attributed to excessively large distillation intervals, which potentially lead to a deterioration in the quality of global consensus, thereby compromising the original convergence process of the agents.

\clearpage

\subsection{Performance improvement efficacy of FedHPD under different public state sets}\label{Appendix: FedHPD under different public state sets}

To investigate the impact of public dataset on FedHPD, we select three groups of public state sets of the same size but with significant differences.
We display the data distribution after dimensionality reduction through principal component analysis in Fig.~\ref{pic:Distributions of different public datasets.}, where the differences between different state sets can be visually observed.
From the system comparison in Fig.~\ref{pic:Comparisons of system performance with different datasets.} and the individual comparison in Fig.~\ref{pic:Comparisons of individual agent with different datasets.}, we conclude that changes in the public state set have almost no impact on the improvement effect of FedHPD.

\begin{figure*}[h]	
	\centerline{\includegraphics[width=0.45\textwidth]{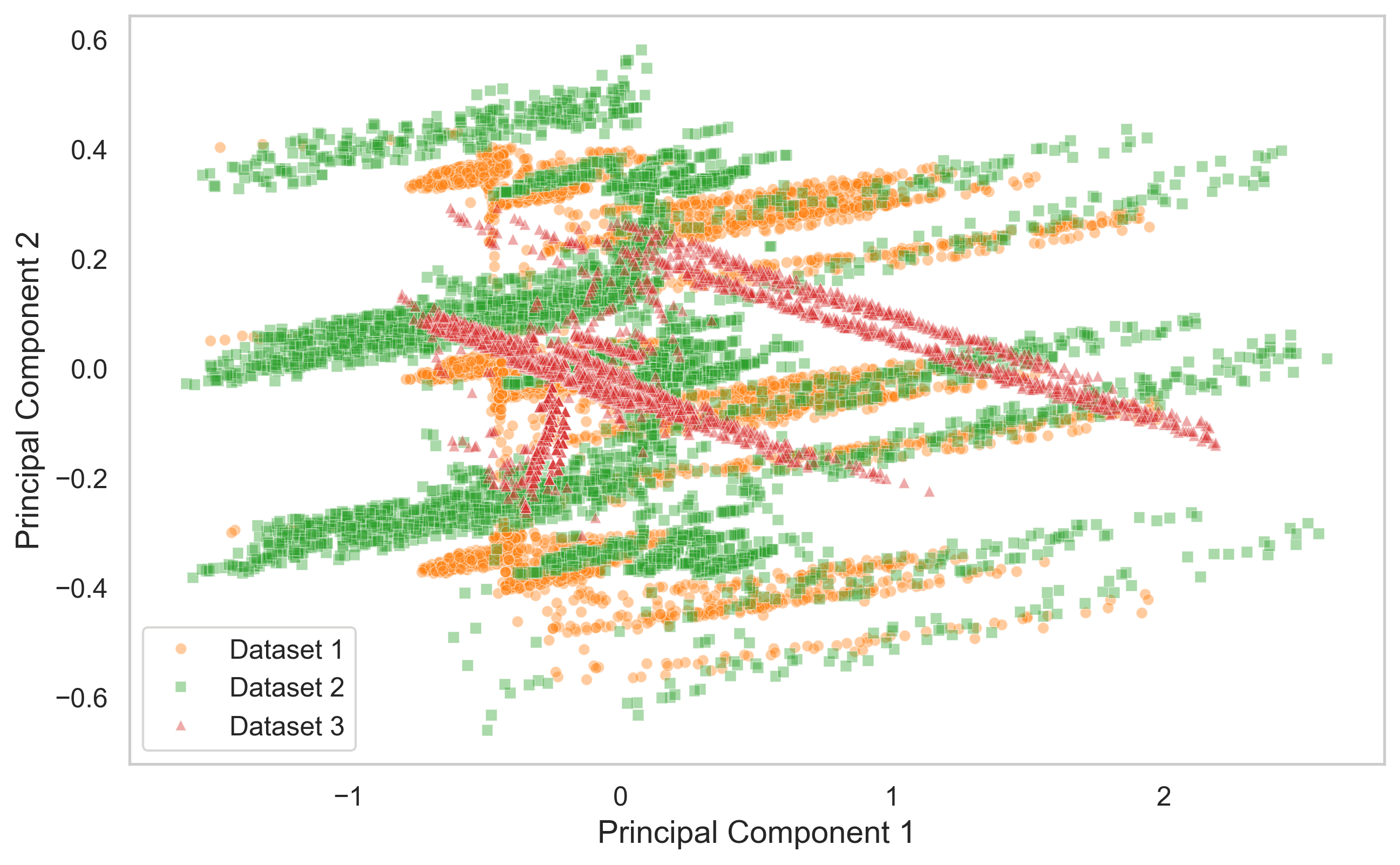}}
	\caption{Distributions of different public datasets (Employing PCA for dimensionality reduction).}
    \label{pic:Distributions of different public datasets.}
\end{figure*} 

\begin{figure*}[h]	
	\centerline{\includegraphics[width=0.4\textwidth]{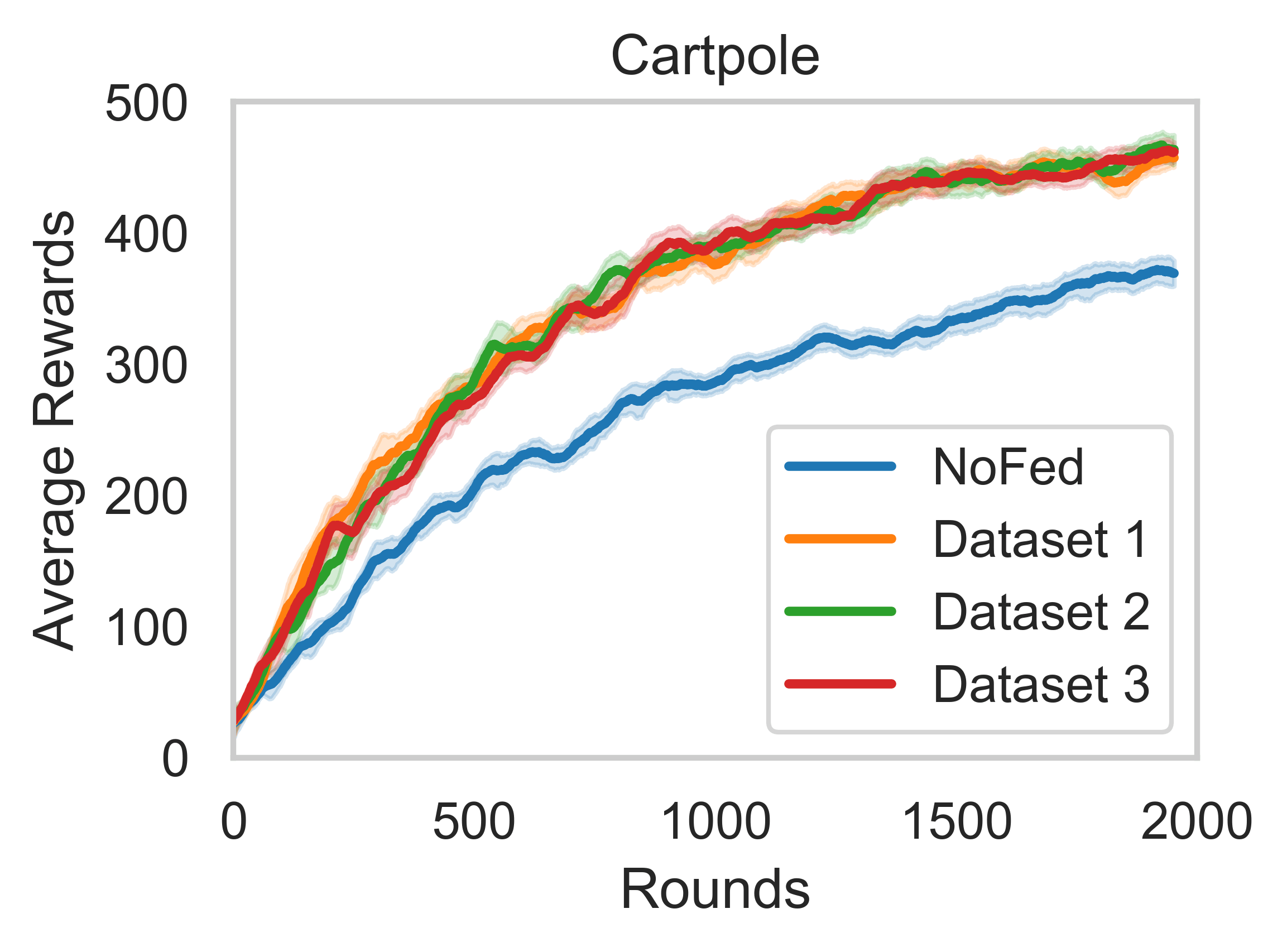}}
	\caption{Comparisons of system performance with different datasets in Cartpole ($d$ = 5).}
    \label{pic:Comparisons of system performance with different datasets.}
\end{figure*} 

\begin{figure*}[h]	
	\centerline{\includegraphics[width=\textwidth]{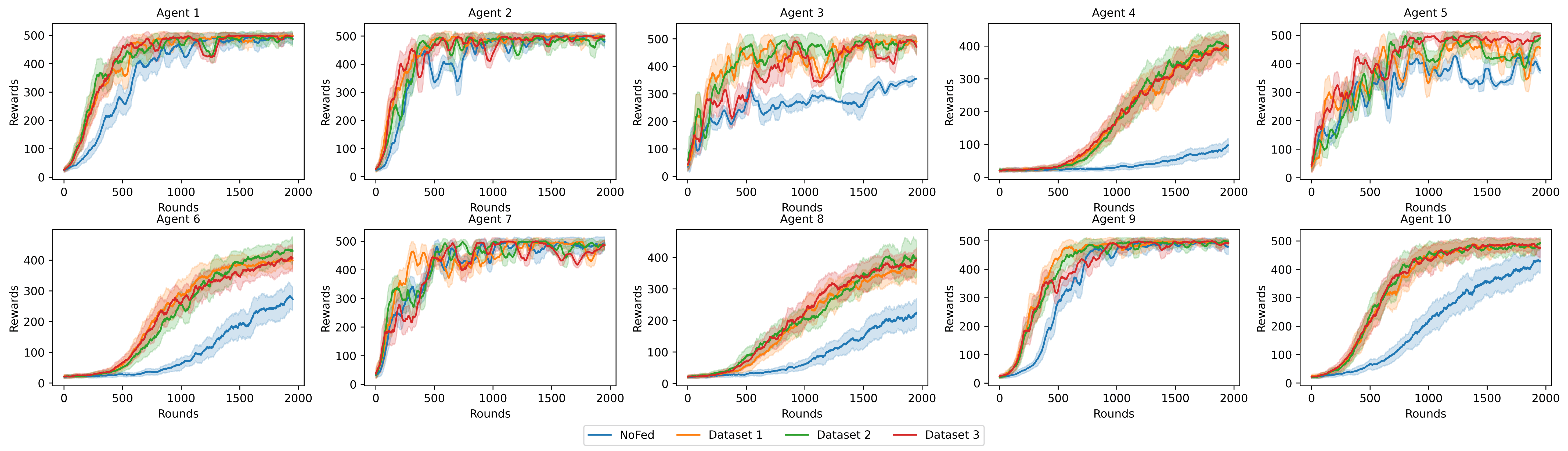}}
	\caption{Comparisons of individual agent with different datasets in Cartpole ($d$ = 5).}
    \label{pic:Comparisons of individual agent with different datasets.}
\end{figure*}

\end{document}